%% file: arxiv.tex
\documentclass{article}

\PassOptionsToPackage{numbers, compress}{natbib}
\usepackage{fullpage}

\usepackage{amsmath,amsfonts,amsthm,amssymb} 
\usepackage[utf8]{inputenc} 
\usepackage[T1]{fontenc}    

\usepackage{hyperref}       
\usepackage{wrapfig}
\usepackage{url}            
\usepackage{booktabs}       
\usepackage{nicefrac}       
\usepackage{microtype}      
\usepackage{multirow}
\usepackage{multicol}
\usepackage{soul}

\usepackage{xcolor}
\usepackage{hyperref}
\usepackage[bbgreekl]{mathbbol}

\newcommand{\Blue}[1]{{\color{blue}#1}}
\newcommand{\Red}[1]{{\color{red}#1}}

\usepackage[makeroom]{cancel}
\usepackage[utf8]{inputenc}

\DeclareMathOperator*{\argmax}{arg\,max}

\usepackage{hyperref}
\usepackage{url}
\usepackage[titletoc,title]{appendix}
\usepackage{amsfonts}
\usepackage{graphicx}
\usepackage{algorithm}
\usepackage{algorithmic}
\usepackage{multirow}
\usepackage{hhline}
\usepackage{booktabs}
\usepackage{comment}

\newcommand{\Xcal}{{\mathcal{X}}}

\newcommand{\Wcal}{{\mathcal{W}}}

\newcommand{\supp}{{\text{supp}}}
\newcommand{\bp}{{\mathbf{p}}}
\newcommand{\bg}{{\mathbf{g}}} 
\newcommand{\bx}{{\mathbf{x}}} \newcommand{\bv}{{\mathbf{v}}} 
\newcommand{\be}{{\mathbf{e}}} \newcommand{\bm}{{\mathbf{m}}} 
\newcommand{\bw}{{\mathbf{w}}} 
 
\newcommand{\bo}{{\mathbf{o}}} 
 
\newcommand{\bc}{{\mathbf{c}}} 
 \newcommand{\bl}{{\mathbf{l}}} 
\usepackage[font=small,labelfont=bf]{caption}

\usepackage{comment}
\newtheorem{definition}{Definition}
\newtheorem{problem}{Problem}

\newtheorem{claim}{Claim}
\newtheorem{lemma}{Lemma}
\newtheorem{remark}{Remark}

\newtheorem{theorem}{Theorem}
\newtheorem{proposition}{Proposition}

\usepackage{lineno}

\newcommand{\shortrnote}[1]{ &  & \text{\footnotesize\llap{#1}}}
\newcommand{\longrnote}[1]{ &  &  \\   &  &  &  &  & \notag \text{\footnotesize\llap{#1}}}

\makeatletter
\DeclareOldFontCommand{\rm}{\normalfont\rmfamily}{\mathrm}
\DeclareOldFontCommand{\sf}{\normalfont\sffamily}{\mathsf}
\DeclareOldFontCommand{\tt}{\normalfont\ttfamily}{\mathtt}
\DeclareOldFontCommand{\bf}{\normalfont\bfseries}{\mathbf}
\DeclareOldFontCommand{\it}{\normalfont\itshape}{\mathit}
\DeclareOldFontCommand{\sl}{\normalfont\slshape}{\@nomath\sl}
\DeclareOldFontCommand{\sc}{\normalfont\scshape}{\@nomath\sc}
\makeatother
\usepackage[utf8]{inputenc} 
\usepackage[T1]{fontenc}    
\usepackage{hyperref}       
\usepackage{url}            
\usepackage{booktabs}       
\usepackage{amsfonts}       
\usepackage{nicefrac}       
\usepackage{microtype}      
\usepackage{authblk}

\newcommand{\R}{\mathbb{R}}

\title{Discrete Adversarial Attacks and Submodular Optimization\\ with Applications to Text Classification}
\author[$\star\thanks{Both authors contributed equally to this work}$]{Qi Lei}
\author[$\dagger^\ast$]{Lingfei Wu}
\author[$\dagger$]{Pin-Yu Chen}
\author[$\star$]{Alexandros G. Dimakis}
\author[$\star\ddagger$]{Inderjit S. Dhillon}
\author[$\dagger$]{Michael Witbrock}
\affil[ ]{$^\star$ UT Austin \hspace{1cm} $^\dagger$ IBM Research \hspace{1cm} $^\ddagger$ Amazon}
\affil[ ]{\texttt{\{leiqi@ices, dimakis@austin, dhillon@cs\}.utexas.edu} }
\affil[ ]{\texttt{\{wuli@us., pin-yu.chen@, witbrock@us.\}ibm.com }}

\begin{document}

\maketitle

\input{abstract.tex}

\input{contents_arxiv.tex}
\bibliographystyle{IEEEtran}
\bibliography{IEEEabrv,rethink_textadv}

\include{appendix_arxiv}

\end{document}

%% file: abstract.tex
\begin{abstract}
Adversarial examples are carefully constructed modifications to an input that completely change the output of a classifier but are imperceptible to humans. Despite these successful attacks for continuous data (such as image and audio samples), 
generating adversarial examples for discrete structures such as text has proven significantly more challenging. 
In this paper we formulate the attacks with discrete input on a set function as an optimization task. We prove that this set function is submodular for some popular neural network text classifiers under simplifying assumption. This finding guarantees a $1-1/e$ approximation factor for attacks that use the greedy algorithm. Meanwhile, we show how to use the gradient of the attacked classifier to guide the greedy search. Empirical studies with our proposed optimization scheme show significantly improved attack ability and efficiency, on three different text classification tasks over various baselines. We also use a joint sentence and word paraphrasing technique to maintain the original semantics and syntax of the text.
This is validated by a human subject evaluation in subjective metrics on the quality and semantic coherence of our generated adversarial text. 
\end{abstract}

%% file: contents_arxiv.tex
\section{Introduction}

Adversarial examples are carefully constructed modifications to an input that completely change the output of a classifier but are imperceptible to humans. Spam filtering and the carefully-crafted emails designed to fool these early classifiers are the first examples of adversarial machine learning going back to 2004~~\cite{dalvi2004adversarial,lowd2005adversarial}; see also the comprehensive survey by Biggio et al.~~\cite{biggio2017wild}. 
Szegedy et al.~~\cite{szegedy2013intriguing} discovered that deep neural network image classifiers can be fooled with tiny pixel perturbations; exploration of this failure of robustness has received significant attention recently, see e.g. ~~\cite{goodfellow2014explaining,moosavi2016deepfool,papernot2016distillation,carlini2017towards,evtimov2017robust,chen2017zoo,chen2017ead,su2018robustness}.
 Adversarial training~~\cite{goodfellow2014explaining,madry2017towards} seems to be the state of the art in defense against adversarial attacks, but 
 creating robust classifiers remains challenging, especially for large image classifiers, see e.g. Athalye at al.~~\cite{athalye2018obfuscated}. 

Despite these successful attacks for continuous data (such as image and audio samples), 
generating adversarial examples for discrete structures such as text and code has proven significantly more challenging in two aspects:

One challenge is how to develop a fast yet (provably) effective attacking scheme. Gradient-based adversarial attacks for continuous data no longer directly apply to discrete structures. Although some variants are proposed when the model is differentiable to the embedding layer~\cite{papernot2016crafting,li2016understanding,ebrahimi2017hotflip,gong2018adversarial}, this line of methods achieve efficiency but suffer from poor success rate. \\
Meanwhile, another natural idea is to find feasible replacement for individual features like words or characters. However, since the space of possible combinations of substitutions  grows exponentially
with the length of input data, finding the optimal combination of substitutions is intractable. Recent heuristic attacks on NLP classifiers operate by greedy character-level or word-level replacements ~\cite{ebrahimi2017hotflip,kuleshov2018adversarial,yang2018greedy}.
However, greedy methods are usually slow, and it's theoretically not understood when they achieve good performance. 

The other issue is how to maintain the original functionality of the input. Specifically for text, it remains challenging to preserve semantic and syntactic properties of the original input from the point of view of a human. Existing methods either require to change too many features, or change the original meaning. For instance, \cite{kuleshov2018adversarial} alters up to $50\%$ of words in each input document to achieve a $30\%$ success rate. \cite{gong2018adversarial} attacks the document by replacing with completely different words. \cite{jia2017adversarial} inserts irrelevant sentences to the original text. Such changes can be easily detected by humans. 


In this paper we argue that these limitations can be 
be resolved with the framework we propose. We highlight our main contributions as follows:

We propose a general framework for discrete attacks. We apply our framework to designing adversarial attacks for text classifiers but our techniques can be applied more broadly. For instance, the attacks include but are not limited to malware detection, spam filtering, or even discrete attacks defined on continuous data, e.g., segmentation of an image.

We formulate the attacks with discrete input on a set function  as an optimization task. This problem, however, is provably NP-hard even for convex classifiers. We unify existing gradient-based as well as greedy methods using a general combinatorial optimization via further assumptions. We note that gradient methods solve a relaxed problem in polynomial time; while greedy algorithm for creating attacks has a provable $1-1/e$ approximation factor assuming the set function is submodular.
We theoretically show that for two natural classes of neural network text classifiers, the set functions defined by the attacks are submodular. We specifically analyze two classes of classifiers: The first is word-level CNN without dropout or softmax layers. The second is a recurrent neural network (RNN) with one-dimensional hidden units and arbitrary time steps. 

Nevertheless, greedy methods can be very time consuming when the space of attacks is large. We show how to use the gradient of the attacked classifier to guide the combinatorial search. Our proposed gradient-guided greedy method is inspired by the greedy coordinate descent Gauss-Southwell rule from continuous optimization theory.
The key idea is that we use the magnitude of the gradient to decide which features to attack in a greedy fashion.

We extensively validate the proposed attacks empirically. With the proposed optimization scheme, we show significantly improved attack performance over
most recent baselines. Meanwhile we propose a joint sentence and word paraphrasing technique to simultaneously ensure retention of the semantics and syntax of the text.



\begin{figure*}[htb]
\label{fig:example}
\footnotesize{
\noindent\fbox{
\parbox{\textwidth}{
Task: Sentiment Analysis. Classifier: LSTM. Original: 100\% Positive. ADV label: 100\% Negative. 
}
}
\noindent\fbox{
\parbox{\textwidth}{
I suppose I should write a review here since my little Noodle-oo is currently serving as their spokes dog in the photos. We both love Scooby Do's. They treat my little butt-faced dog like a prince and are receptive to correcting anything about the cut that I perceive as being weird. Like that funny poofy pompadour. Mohawk it out, yo. Done. In like five seconds my little man was looking fabulous and bad ass. Not something easily accomplished with a prancing pup that literally chases butterflies through tall grasses. (He ended up looking like a little lamb as the cut grew out too. So adorable.)  The shampoo they use here is also amazing. Noodles usually smells like tacos (a combination of beef stank and corn chips) but after getting back from the Do's, he smelled like Christmas morning! Sugar and spice and everything nice instead of frogs and snails and puppy dog tails. He's got some gender identity issues to deal with.   \st{The pricing is also cheaper than some of the big name conglomerates out there} \Red{The price is cheaper than some of the big names below}. I'm talking to you Petsmart! I've taken my other pup to Smelly Dog before, but unless I need dog sitting play time after the cut, I'll go with Scooby's. They genuinely seem to like my little Noodle monster.
}
}
\noindent\fbox{
\parbox{\textwidth}{
Task: Fake-News Detection. Classifier: LSTM. Original label: 100\% Fake. ADV label: 77\% Real}
}
\noindent\fbox{
\parbox{\textwidth}{
\st{Man} \Blue{Guy} punctuates high-speed chase with stop at In-N-Out Burger drive-thru Print  [Ed.\st{ - Well, that's} \Red{Okay, that 's} a new one.]  \st{A} \Blue{One} man is in custody after leading police on a bizarre chase into the east Valley on Wednesday night.  Phoenix police \st{began} \Red{has begun} following the suspect in Phoenix and the pursuit continue\st{d} into the east Valley, but it took a bizarre turn when the suspect stopped at an In-N-Out Burger restaurant’s \st{drive-thru} \Blue{drive-through} near Priest and Ray Roads in Chandler.  The suspect appeared to order food, but then drove away and got out of his pickup truck near Rock Wren Way and Ray Road.  He \st{then ran into a backyard} \Red{ran to the backyard} and tried to \st{get into a house through the back door} \Red{get in the home}.}
}
\noindent\fbox{
\parbox{\textwidth}{
Task: Spam Filtering. Classifier: WCNN. Original label: 100\% None-spam. ADV label: 100\% Spam}
}
\noindent\fbox{
\parbox{\textwidth}{
> > Hi All, \\
> > I'm new to R from a C and Octave/Matlab background.  \st{I am trying to > > construct} \Red{I 'm trying to build} some classes in R to which I want to attach \st{pieces of} data. \\
> > First, is attr(obj, 'member name')  > > this?  > > No, it isn't. You seem to be trying to deduce new-style classes from a > > representation used before R 2.4,  (actually, still used)  > > but in any case it would not be >> sensible. \st{Please consult} \Red{Contact} John M. Chambers. Programming with Data. > > Springer, New York, 1998, and/or William N. Venables and Brian D. Ripley. > > S Programming. Springer, New York, 2000, or for a shorter online resource: > >  http://www.stat.auckland.ac.nz/S-Workshop/Gentleman/Methods.pdf > > Unfortunately, all of those references are at least 4 years out of > date when it comes to S4 methods.  Is there any comprehensive > reference of the current implementation of the S4 OO system apart from > the source code?  Not that I know of, and \st{it is} \Red{it's} a moving target.  (E.g. I asked recently  about some anomalies in the S4 bit introduced for 2.4.0 and what the  intended semantics are.)  I've said before that I believe we can only  help solve some of the efficiency issues with S4 if we have a technical  manual.  It is unfair to pick out S4 here, but the 'R Internals' manual is an  attempt to document important implement\st{ation} \Blue{ing} details (mainly by studying  the code), and that has only got most of the way through src/main/*.c. 
    }
}
}
\caption{Examples of generated adversarial examples. The color red denotes sentence-level paraphrasing, and blue denotes word-level paraphrasing. }
\label{intro:example}
\end{figure*}

\section{Related Work}
Broadly speaking, adversarial examples refer to minimally modified natural examples that are spurious but perceptually similar and that lead to 
inconsistent decision making between humans and machine learning models. An example is automatically classifying an adversarial stop sign image (according to humans) as a speed limit sign. For continuous data such as images or audio, generating adversarial examples is often accomplished by crafting additive perturbations of natural examples,
resulting in visually imperceptible or inaudible noise that misleads a target machine learning model. These small yet effective perturbations are difficult for humans to detect, but will cause an apparently well-trained machine learning model to misbehave; in particular, neural networks have been shown to be susceptible to such attacks ~\cite{szegedy2013intriguing}, giving rise to substantial concern about safety-critical and security-centric machine learning applications.

For classifiers with discrete input structures, a simple approach for generating adversarial examples is to replace each feature with similar alternatives. Such features for text classification tasks are usually individual words or characters. Such attacks can be achieved using continuous word embeddings or with respect to some designed score function; this approach has been applied to attack
NLP classifiers~~\cite{papernot2016crafting,li2016understanding,miyato2016adversarial,samanta2017towards,liang2017deep,yao2017automated,gong2018adversarial,kuleshov2018adversarial,gao2018black,alzantot2018generating,yang2018greedy} and sequence-to-sequence models ~\cite{ebrahimi2017hotflip,wong2017dancin,zhao2017generating,cheng2018seq2sick}. The work in \cite{ribeiro2018semantically} considers semantically equivalent rules for debugging  NLP models, but under the same input structure.
This is a natural but limited practice to only consider attacks within one input structure, namely word or characters, but no joint attacks, nor the effect incurred from sentences.
Unlike prior work, we conduct a joint sentence and word paraphrasing technique. It considers sentence-level factors and allows more degrees of freedom in generating text adversarial examples, by exploring the rich set of semantically similar paraphrased sentences.

Jia and Liang studied adversarial examples in reading comprehension systems by inserting additional sentences ~\cite{jia2017adversarial}, which is beyond the concept of this paper since the approach changes the original meanings.
Another related line of research, although not cast as adversarial examples, focuses on improving  model robustness against  out-of-vocabulary terms ~\cite{belinkov2017synthetic} or obscured embedding space representations ~\cite{mrkvsic2016counter}.  


\section{Preliminary}
In this paper, we propose a general framework for generating adversarial examples with discrete input data. A collection of such data and corresponding attacks are presented in Table \ref{tab:general_tasks}.

To present our mathematical formulation, we start by introducing some notation. 

{\bf Input Structure.} 
Let the input $\bx=[x_1,x_2,\cdots,x_n]\in \Xcal^n$ be a list of $n$ features (might be padded). For text environment, the feature space $\Xcal$ can be the character, word, phrase, or sentence space. For the problem of malware detection, $\bx$ is a concatenation of code pieces.

  \begin{minipage}{\textwidth}
  \hspace{-0.5cm}
    \begin{minipage}[b]{0.4\textwidth}
    \centering
\begin{tabular}{c|c}
        \toprule
        input data & task\\
        \hhline{==}
        document & text classification \\
        \midrule
        code & malware detection \\
\midrule
url address & malicious website check\\
\bottomrule
    \end{tabular}
      \captionof{table}{Applications to the framework.}
      \label{tab:general_tasks}
    \end{minipage}
  \begin{minipage}[b]{0.55\textwidth}
    \centering
    \includegraphics[width=\textwidth]{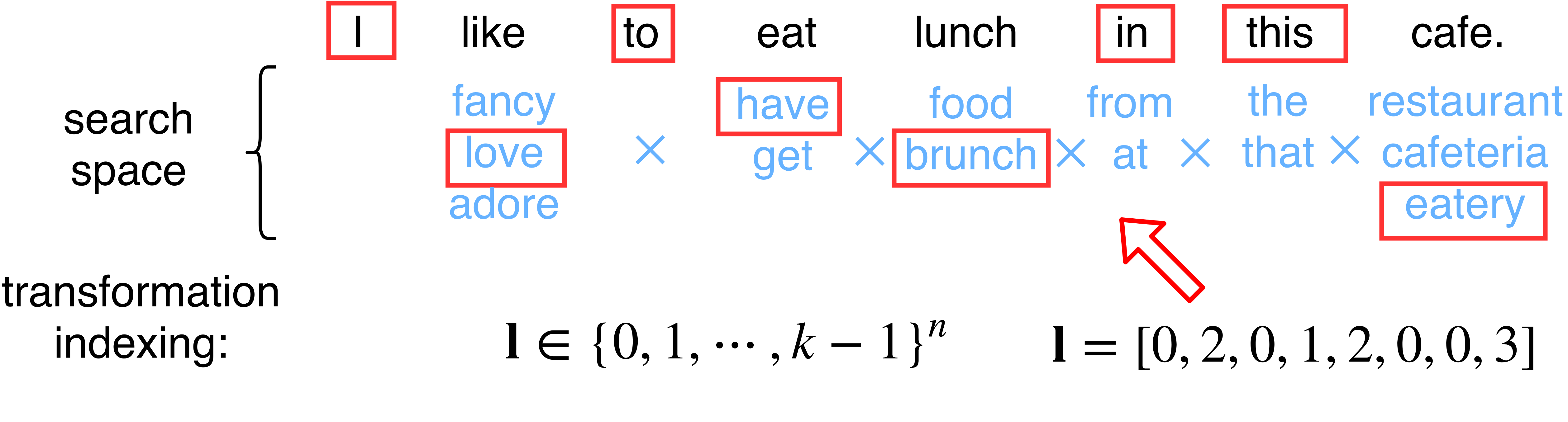}
    \captionof{figure}{An illustration of the transformation indexing when applying to a text sentence. In this example, the transformation denoted as $\bl$ modifies the original sentence to the new one shown in the red boxes. }
    \label{fig:illustration}
  \end{minipage}
  \end{minipage}

\begin{remark}
\label{remark:sentence_space}
For concrete usage, we use $w\in \Wcal$ to denote word space, and $s\in\mathcal{S}$ to denote sentences to distinguish the differences. 
\end{remark}


{\bf Embedding $V$.} The embedding layer is a key transition from discrete input data into continuous space, which could then be fed into the classifier. For text domain, we typically use the bag-of-words embedding or word-to-vector embedding.

For a bag-of-words embedding, $V: \Xcal^n \rightarrow \R^D$ represents a document as the statistics of word counts, i.e., the summation of each word's one-hot representation. Meanwhile, word-to-vector embeddings characterize different words as $D$-dimensional vectors, i.e., $V(x)\in \R^D, \forall x\in \Xcal$.
When there's no ambiguity, we also use $V: \Xcal^n \rightarrow \R^{n\times D}$ to denote the concatenation of word vectors of the input document as a list of words.

{\bf Transformation Indexing.} 
 Suppose each feature $x\in\Xcal$ has (at most) $k-1$ possible replacements, denoted by $x^{(i)}, i\in [k-1] (\equiv \{1,2,\cdots , k-1\})$. For future use, we also define $x^{(0)}=x, \forall x\in\Xcal$.\\
A valid transformation $T$ is the combined replacement of each individual feature $x_i, i\in [n]$. Therefore we index $T$ by a vector $\bl\in \{0,1,\cdots, k-1\}^n$, and $l_i$ indicates the index of each replacement $i$. Namely, $T_{\bl}(\bx=[x_1,x_2,\cdots, x_n])=[x_1^{(l_1)}, x_2^{(l_2)},\cdots, x_n^{(l_n)}]$.
An example with word replacement in the text classification environment can be found in Figure \ref{fig:illustration}.

{\bf Classifier output $C_y$.} We consider a targeted attack, i.e., we want to maximize the output probability $C$ over a specific target label $y$. 

In this paper, we use a regular lower-case symbol to denote a scalar or a single feature, and use a bold lower-case symbol for a vector or a list of features.
\subsection{Problem Setup}
\label{sec:formulation}
In most scenarios, we only allow transformations on at most $m$ features, then the constraint is $\|\bl\|_0\leq m$. Therefore we present the adversarial attack problem formally:
\begin{problem}
\label{problem:origin}
For some input data $\bx\in \Xcal^n$ and target label $y$, we try to find a feasible transformation $T_{\bl^*}$, where $\bl^*\in \{0,1,\cdots k-1\}^n$ is the index so that:
\begin{equation}
\bl^*=\argmax_{\|\bl\|_0\leq m} C_{y} \left(V\left(T_{\bl}(\bx)\right)\right).    
\end{equation}
Or similarly, we want to find the set of features to attack, i.e., 
\begin{equation}
\label{eqn:origin}
S^*=\argmax_{|S|\leq m} f(S),    
\end{equation}
where we defined the set function $f:2^{[n]}\rightarrow \R, $ $f(S)= \max_{\supp(\bl)\subset S} C_{y}(V(T_{\bl}(\bx))) $.
\end{problem}
The set function $f(S)$ represents the classifier output for the target label $y$ if we apply a set of transformations $S$.  We are therefore searching over all possible sets of up to $m$ replacements to maximize the probability of the target label output of a classifier. 

\begin{remark}
In this paper, we focus on replacements via word and sentence paraphrasing for empirical studies. However, our formulation is general enough to represent any set of discrete transformations. Possible transformations include replacement with the nearest neighbor of the gradient direction~\cite{gong2018adversarial} and word vectors~\cite{kuleshov2018adversarial}, or flipping characters within each word ~\cite{ebrahimi2017hotflip}. We will also conduct a thorough experimental comparisons among different choices.
\end{remark}

\section{Theoretical Analysis}
First, notice that the original problem is computationally intractable in general:
\begin{proposition}
\label{remark:subsetsum}
For a general classifier $C_y$, the problem \ref{problem:origin} is NP-hard. Specifically, even for some convex $C_{y}$, the problem \ref{problem:origin} can be polynomially reduced from subset sum and hence is NP-hard.
\end{proposition}
Details and all proofs referenced to in this paper can be found in the appendix.

\subsection{Unifying Related Methodology via Further Assumptions}
Fortunately, with further assumptions it becomes possible to solve problem \ref{problem:origin}, above, in polynomial time. Some existing heuristics are proposed to generate adversarial examples for the text classification problem. Though usually not specifically proposed in the relevant literature, we unify the underlying assumptions for these heuristics to succeed in polynomial time in this section.

One possible assumption is that the original function $C_{y}$ is smooth, which could afterwards be approximated by its first-order Taylor expansion:
\begin{eqnarray*}
C_{y}(V(T_{\bl}(\bx)))&=&C_{y}(\bv)+\langle \nabla C_{y}(\bv),V(T_{\bl}(\bx))-\bv\rangle + \mathcal{O}\left(\|V(T_{\bl}(\bx))-\bv\|_2^2\right)
\end{eqnarray*}
where $\bv=V(\bx)$.
Therefore, Problem \ref{problem:origin} can be relaxed as follows:
\begin{problem}
\label{problem:frank-wolfe}
Given gradient $\nabla C_{y}(\bv)$, where $\bv=V(\bx)$, maximize function $C_{y}$ by its first-order Taylor expansion:
\begin{equation}
\label{eqn:frank_wolfe}
\bl^*=\argmax_{\|\bl\|_0\leq m} V(T_{\bl}(\bx))^\top \nabla C_{y}(\bv).
\end{equation}
\end{problem}
Problem \ref{problem:frank-wolfe} is similar to the Frank-Wolfe method ~\cite{frank1956algorithm} in continuous optimization and is easy to solve: 
\begin{proposition}
\label{lemma:frank_wolf_P}
Problem \ref{problem:frank-wolfe} can be solved in polynomial time for both bag-of-words and word to vector embeddings. Specifically, $f(S)=\argmax_{\supp(\bl)\subset S} V(T_\bl(\bx))^\top \Delta C_y(\bv)$ can be written as $\sum_{i\in S}w_i$ for some $w$ irrelevant to $S$, where $\bv=V(\bx)$.
\end{proposition}

Related methods like ~\cite{gong2018adversarial} are attempts to solve problem \ref{problem:frank-wolfe}. They propose to conduct transformations via replacement by synonyms chosen by \eqref{eqn:frank_wolfe}. However, activations like ReLU break the smoothness of the function, and first order Taylor expansion only cares about very local information, while embeddings for word synonyms could be actually not that close to each other. Consequently, this unnatural assumption prevents related gradient-based attacks to achieve good performance.

Besides smoothness, another more natural assumption is that $f(S)$ in the original problem \ref{problem:origin} is submodular~~\cite{narayanan1997submodular,fujishige2005submodular}. Submodular is a property that is defined for set functions, which characterizes the diminishing returns of the function value change as the size of the input set increases. 
\begin{definition}\cite{schrijver2003combinatorial}~~
If $\Omega$ is a finite set, a submodular function is a set function $ f:2^{\Omega }\rightarrow \mathbb{R}$, where $2^{\Omega }$  denotes the power set of $\Omega$ , which satisfies one of the following equivalent conditions.
\begin{enumerate}
    \item For every $X, Y \subseteq \Omega$ with $ X \subseteq Y$ and every $ x\in \Omega \setminus Y$ we have that $ f(X\cup \{x\})-f(X)\geq f(Y\cup \{x\})-f(Y)$.
\item For every $S, T \subseteq \Omega$ we have that $ f(S)+f(T)\geq f(S\cup T)+f(S\cap T)$.
\item For every $X\subseteq \Omega$ and $x_1,x_2\in \Omega\backslash X$ we have that $f(X\cup \{x_1\})+f(X\cup \{x_2\})\geq f(X\cup \{x_1,x_2\})+f(X)$.
\end{enumerate}
\end{definition}
With the design of $f(S)$ in Problem \ref{problem:origin} to be monotone non-decreasing and if we further assume $f$ to be submodular, our task becomes to maximize a monotone submodular function subject to a cardinality constraint ~\cite{nemhauser1978analysis}. Therefore, greedy method guarantees a good approximation of the optimal value of Problem \ref{problem:origin}:
\begin{claim}
\label{lemma:monotone}
In problem \ref{problem:origin}, $f$ is monotone non-decreasing. Furthermore, if the function $f$ is submodular, greedy methods achieve a $(1-1/e)$-approximation of the optimal solution in polynomial time.
\end{claim}

Both our work and the optimization scheme from ~\cite{kuleshov2018adversarial} propose some variants of greedy methods with the underlying submodular assumption.  

The greedy method proposed in ~\cite{kuleshov2018adversarial} selects candidate replacements directly by function value, one word at a time, which we will refer as the objective-guided greedy method. We will propose a more efficient yet comparable effective greedy method that is guided by the gradient magnitude in Section \ref{sec:GGGM}, and compare with the above two methods in Section \ref{sec:empirical}. As an extension from the continuous optimization, our method uses the well-studied Gauss-Southwell rule~~\cite{nutini2015coordinate} that is provably better than random selection. In each iteration, we determine and select the most important words by the gradient norm of words' embeddings, and then find the greediest transformation within the search space of the selected words. The advantage is that we are able to conduct multiple replacements in one iteration and thus take into consideration the joint effect of multiple words replacements.
We will introduce our method, which we call Gradient-Guided Greedy Word Paraphrasing in Algorithm \ref{alg:word}, and will show empirical performance comparison with the (objective-guided) greedy method ~\cite{kuleshov2018adversarial} and the gradient method used in ~\cite{gong2018adversarial} in Section \ref{sec:empirical}. 

\subsection{Submodular Neural Networks on the Set of Attacks}
To argue that submodular is a natural assumption, we study and summarize the neural networks are submodular on the set of attacks. 

In \cite{bilmes2017deep}, it provides a class of submodular functions used in the deep learning community called deep submodular functions.
Nevertheless the deep submodular functions are not necessarily applicable to our set function. We hereby formally prove the following two kinds of neural networks, that are ubiquitously used for text classification, indeed satisfy submodular property on the set of attacks under some conditions.

\begin{figure}
    \centering
    \includegraphics[width=0.65\columnwidth]{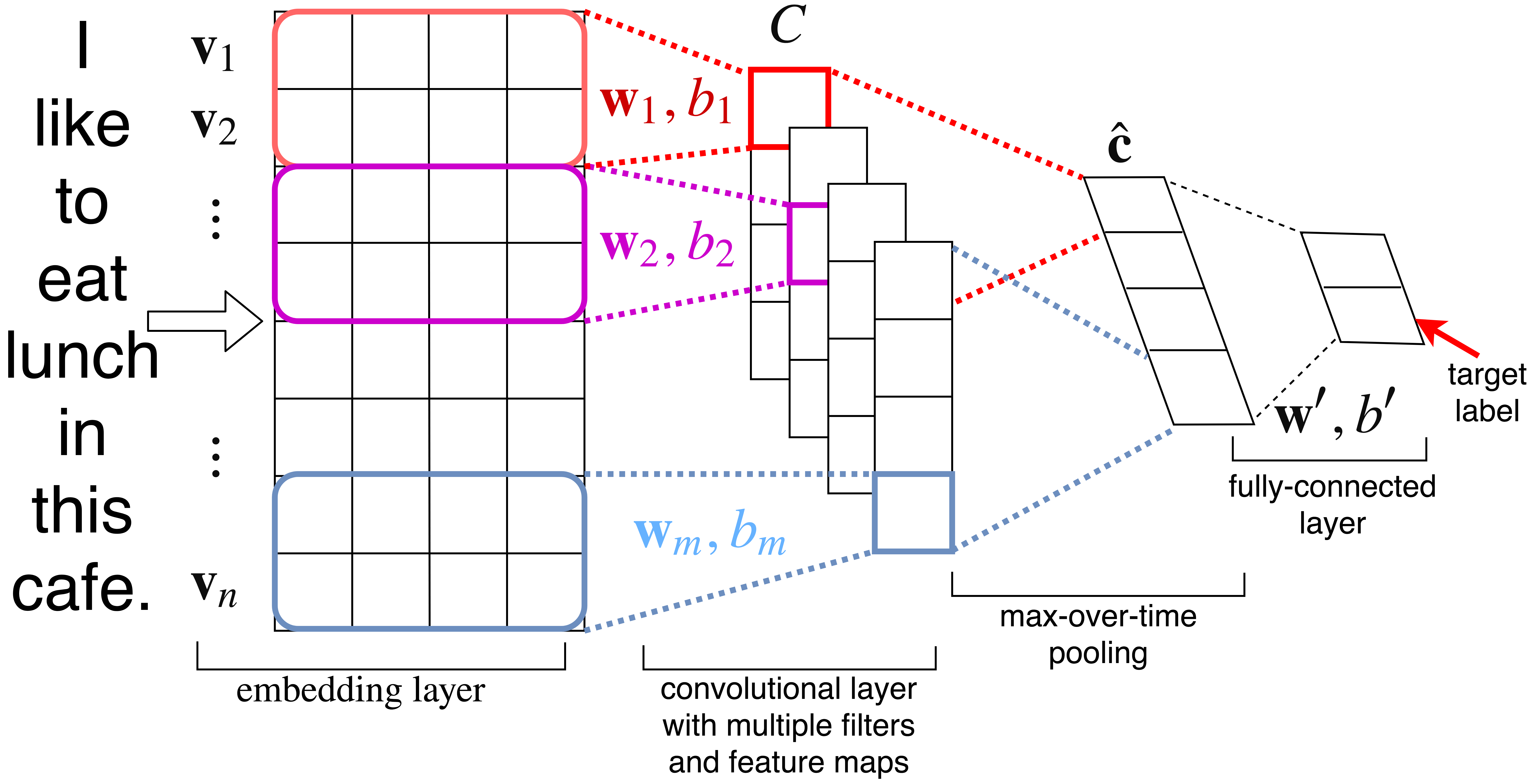}
    \caption{Model architecture of simplified W-CNN for an example sentence.}
    \label{fig:my_label}
\end{figure}

\subsubsection{Simplified W-CNN ~\cite{kim2014convolutional}}
Denote the stride as $s$, the number of grams (window size) $h$, and the word vector of the $i$-th word in a document as $\bv_i$($\equiv V(x_i)$). Then the output for the convolutional layer is a matrix $C=[c_{ij}]_{i\in[n/s],j\in [m]}$ from $n$ words and $m$ filters:
$$c_{ij}=\phi(\bw_j^\top \bv_{s(i-1)+1:s(i-1)+h}+b_j),  ~i=1,2,\cdots n/s,$$
where $\bw_j\in \R^{D h}$ is the $j$-th filter, $b_j$ is the corresponding bias term and $\phi$ is the non-linear, and non-decreasing activation such as ReLU, tanh and sigmoid function. $\bv_{i:j}$ denotes the concatenation of word vectors in the window of words $i$ through $j$, namely $[\bv_i^\top,\bv_{i+1}^\top,\cdots \bv_{j}^\top]^\top \in \R^{D (j-i+1)}$. Each filter $\bw_j$ is applied to individual windows of words to produce a feature map
$\bc^j = [c_{1j},c_{2j},\cdots c_{n/s,j}]^\top$.

Afterwards, a max-over-time pooling is applied to each feature map to form the penultimate
layer $\hat{\bc}=[\hat{c}_1,\hat{c}_2,\cdots \hat{c}_m]$, where $\hat{c}_i$ is the largest value in $\bc^j$:
$$\hat{c}_j=\max_{i}c_{ij}.$$

Compared to the original ~\cite{kim2014convolutional} paper, we only omit the dropout and softmax layer, and instead consider the following WCNN classifier output for a target label:
\begin{equation}
C^{\text{WCNN}}(\bv_{1:n})=\bw'\cdot \hat{\bc}+b'
\label{eqn:wcnn}
\end{equation}

\begin{theorem}
\label{remark:cnn_submodular}
We consider the simple version of W-CNN classifier described in \eqref{eqn:wcnn}, and suppose there's no overlapping between each window, i.e., $s\geq h$, and $\bw'$ has all non-negative values. If further we only look at transformations that will increase the output, i.e., $\bw_j^\top V(x_i^{(t)})\geq \bw_j^\top V(x_i), \forall i\in [n], j\in [m], t\in [k-1]$,
then $f^{\text{WCNN}}(S)=\max_{\supp(\bl)\in S} C^{\text{WCNN}}(V(T_{\bl}(x)))$ is submodular. 
\end{theorem}
The proof sketch is as follows. Every coordinate in $\hat{\bc}$ is a combination of max pooling over a modular function and is therefore submodular. And finally sums of submodular functions is still submodular. 

Besides word-level CNN, another network that is popular in the NLP community is the recurrent neural network (RNN) or its variants. We will show that under some conditions, RNN satisfies submodular property.

\subsubsection{Recurrent Neural Network with One-dimensional Hidden Units}
Consider a RNN with $T$ time steps and each hidden layer is a single node. Then for all $t\leq T$, given the value of a previous hidden state $h_{t-1}\in\R$ and an input word vector $\bv_{t-1}\in\R^{D}$ ($\bv_{t}\equiv V(x_t)$),
RNN computes the next hidden state $h_{t}$ and output vector $\bo_{t}\in\R$ as:
\begin{align}
\label{eqn:rnn}
  h_t &= \phi(wh_{t-1}+ \bm^\top\bv_{t-1}+b) 
\end{align}
The classifier output is $C^{\text{RNN}}(\bv_{1:T})=yh_T$.

\begin{theorem}
\label{remark:rnn_submodular}
For a recurrent neural network with $T$ time
steps and one-dimensional hidden nodes described in \eqref{eqn:rnn}, if $w$ and $y$ are positive, and the activation is a non-decreasing concave function, then $f^{\text{RNN}}(S)=\max_{\supp(\bl)\in S} C^{\text{RNN}}(V(T_{\bl}(\bx)))$ is submodular. 
\end{theorem}
This result is quite surprising, since the word vectors influence the network's output on different time steps and are by no means separable.
In the proof, we first show that a same amount of change induced on an intermediate layer has a diminishing effect when the network is attacked on more features. Then together with the concavity and non-decreasing property of the network, we are able to finish the proof.

\begin{algorithm*}[bht]
\caption{$Joint~Sentence~And~Word~Paraphrasing(C_{y},\bx^{(0)}, P, \delta, \lambda_s,\lambda_w,\delta_s,\delta_w,\tau,k) $}
\begin{algorithmic}[1]
\STATE {\bfseries Input:} Classifier $C$ associated with target label $y$, input document $\bx^{(0)}$, language model $P$ trained on the training set, syntactic threshold $\delta$, sentence and word paraphrasing ratio $\lambda_s, \lambda_w$, termination threshold $\tau$, WMD threshold $\delta_s, \delta_w$, limit number of paraphrases $k$. 
\STATE Conduct sentence separation $\bx^{(0)}\rightarrow [s_1, s_2,\cdots s_l], s_i\in \mathcal{S}, 1\leq i\leq l$. (See Remark \ref{remark:sentence_space}). 
\STATE Create sentence neighboring set $\mathbf{S}=\{S_1, S_2, \cdots S_l\}$, where each $S_i\subset \mathcal{S}$ satisfies that $|S_i|\leq k$ and $WMD(s_i, s)\leq \delta_s,\forall s\in S_i$.
\STATE $\bx^{(1)}\leftarrow Greedy~Sentence~Paraphrasing(C_{y}, \bx^{(0)}, \mathbf{S}, \lambda_s, \tau)$ in Alg. \ref{alg:sentence}.
\STATE {\bfseries If} $C_{y}(V(\bx))\geq \tau$ {\bfseries Return} $\bx^{(1)}$ 
\STATE Conduct word separation $\bx^{(1)}\rightarrow [w_1, w_2,\cdots w_n], w_i\in \Wcal, 1\leq i\leq n$. 
\STATE Create word neighboring set $\mathbf{W}=\{W_1, W_2, \cdots W_n\}$, where each $W_i\subset \Wcal$  satisfies that $|W_i|\leq k$ and $WMD(w_i, w)\leq \delta_w, |P(\bx^{(1)})-P(\bx'(w))|\leq \delta, \forall w\in W_i$, where $\bx'(w)$ is text $\bx^{(1)}$ in which $w_i$ is  substituted by $w$.
\STATE $\bx^{(2)}\leftarrow$ $Gradient~Guided~Greedy~Word~Paraphrasing(C_{y}, \bx^{(1)},\mathbf{W}, \lambda_w, \tau)$ in Alg. \ref{alg:word}.
\STATE {\bfseries Return} $\bx^{(2)}$
\end{algorithmic} 
\label{alg:joint}
\end{algorithm*}

\begin{algorithm*}[hbt]
\caption{$Greedy~Sentence~Paraphrasing(C_{y}, \bx, \mathbf{S}, \lambda_s, \tau) $}
\begin{algorithmic}[1]
\STATE {\bfseries Input:} Document $\bx$ as list of sentences $[s_1,s_2,\cdots,s_l]$, sentence neighboring sets $\mathbf{S}=\{S_1,S_2\cdots S_n \} $, model $C_{y}$ and parameters $\lambda_s, \tau$.
\WHILE{$C_{y}(V(\bx))\leq \tau$ and number of sentence paraphrased $\leq \lambda_s l$}
\STATE Create candidate set $M=\emptyset $
\FOR{$j=1,2,\cdots, l$}
\FOR{$s \in S_j$}
\STATE Substitute $s_j$ by $s$ to get $x'$ and add it to the candidate set $M\leftarrow M\cup \{\bx'\}$.
\ENDFOR
\STATE $\bx\leftarrow \argmax_{\bx'\in M} C_{y}(V(x')) $
\ENDFOR
\ENDWHILE
\end{algorithmic}
\label{alg:sentence}  
\end{algorithm*}

\section{Adversarial Text Examples via Paraphrasing}
In order to conduct adversarial attacks on models with discrete input data like text, one essential challenge is how to select suitable candidate replacements so that the generated text is both semantic meaning preserving and syntactically valid. Another key issue is how to develop an efficient yet effective optimization scheme to find good transformations. 
To solve the above two issues, we propose our methodology for generating adversarial examples for text.

\subsection{Joint Sentence and Word Paraphrasing}

To coincide with the definition of adversarial examples for text, we first determine appropriate word and sentence paraphrasing methods that maintain the semantic meaning of the original text. Our scheme is to generate an initial set for word and sentence replacements with a well-studied paraphrasing corpus and then filter out discrepant choices based on their semantic and syntactic similarities to the original text. A similar mechanism was also used by \cite{kuleshov2018adversarial} to generate word replacement candidates.

\begin{algorithm*}
\caption{$Gradient~Guided~Greedy~Word~Paraphrasing(C_{y}, \bx, \mathbf{W}, \lambda_w, \tau) $ }
\begin{algorithmic}[1]
\STATE {\bfseries Input}: Document $x$ as a list of words $[w_1,w_2,\cdots, w_n]$, word neighboring sets $\mathbf{W}=\{W_1,W_2\cdots W_n \} $, model $C_{y}$ and parameters $\lambda_w, \tau$.
\STATE Let $N$ (that we set as 5) be the number of words to replace at most in each iteration  
\WHILE{$C_{y}(\bx)\leq \tau$ and number of words paraphrased $\leq \lambda_w n$}
\STATE Compute score for each word $\bp$: $p_i=\|\nabla_i C_{y}(\bv) \|_2$, where $\bv=V(\bx)$ and $\nabla_i$ denotes the gradient with respect to the embedding of the $i$-th word in $x$. 
\STATE Get the indices $I=\{i_1,i_2,\cdots i_N \}$: the $N$ largest indices in $\bp$.
\STATE Create candidate set $M=\{\bx\} $
\FOR{$j\in I$}
\STATE Let the new candidate set $\bar{M}\leftarrow \emptyset$
\FOR{$\bar{\bx} \in M$}
\FOR{$w \in W_j$}
\STATE Substitute the $j$-th word in $\bar{\bx}$ by $w$ to get $\bx'$ and add it to the candidate set $\bar{M}\leftarrow \bar{M}\cup \{\bx'\}$.
\ENDFOR
\ENDFOR
\STATE $M\leftarrow M\cup \bar{M}$
\ENDFOR
\STATE $\bx\leftarrow \argmax_{\bx'\in M} C_{y}(\bx') $
\ENDWHILE
\end{algorithmic}
\label{alg:word}  
\end{algorithm*}

\textbf{Paraphrasing Corpus.} \\
For word paraphrasing, we use the Paragram-SL999 ~\cite{wieting2015paraphrase} of 300 dimensional paragram embeddings to generate neighboring paraphrasing for words. For sentences, we use the pretrained model from Wieting and Gimpel's Para-nmt-50m project ~\cite{wieting-17-millions} to generate sentence paraphrases.

We further specify semantic and syntactic constraints to ensure good quality in adversarial texts: 

{\bf Semantic similarity. } \\
We use the Word Mover Distance (WMD) \cite{kusner2015word} to
measure semantic dissimilarity. For sentence pairs, WMD captures the minimum total semantic distance that the embedded words of one sentence need to “travel” to the embedded words of another sentence. While for words, WMD directly measures the distance between their embeddings.

{\bf Syntactic similarity. }\\
Alongside the semantic constraint, one should also ensure that the generated sentence is fluent and natural. We make use of a language model as in \cite{kuleshov2018adversarial}, $P: \Xcal^n \rightarrow [0,1]$ to calculate the probability of the adversarial sentence, and require: 
$$|\ln(P(\bx))-\ln(P(\bx'))|\leq \delta, $$
where $\bx'$ is the adversarial sentence paraphrased from $\bx$.


In Algorithm \ref{alg:joint}, we present the whole procedure of finding the neighboring sets to conduct our proposal joint sentence and word paraphrasing attack. While with more details, we show how to use the objective value as well as gradient information to guide the search in Algorithm \ref{alg:sentence} (for sentences) and \ref{alg:word} (for words). 

\subsection{Gradient-Guided Greedy Method}
\label{sec:GGGM}
In Section \ref{sec:formulation} we have demonstrated the difficulty of finding the best transformation from combinatorially many choices. 
Here we specify our proposal, gradient-guided greedy word paraphrasing, as shown in Algorithm \ref{alg:word}. We can see that we first use gradient values to determine the index set of $N$ words ($w_{i_1},w_{i_2},\cdots w_{i_N}$) that we want to replace (steps 4-5). Then in steps 7-15 we create a candidate set of all possible transformations in $W_{i_1}\times \cdots \times W_{i_N}$. Finally, we choose the best paraphrase combinations within the candidate set. In this way, we are able to conduct multiple replacements in one iteration and thus take into consideration the joint effect of multiple words replacements. \\
This method is based on an intuition derived from coordinate descent with the Gauss-Southwell rule~\cite{nutini2015coordinate} in the continuous optimization theory; normally, updating the coordinates with the highest absolute gradient values is provably faster than optimizing over random coordinates~\cite{lei2016coordinate,lei2017doubly}. We only conduct this method in word paraphrasing, since the gradient information of sentence embedding is less trustworthy. Usually sentence paraphrasing changes the number of words. The calculated gradient before paraphrasing step might not even correspond to the right position of the new sentence. Therefore it makes more sense to use the objective value only and goes back to our Algorithm \ref{alg:sentence}.


\section{Experiments}
In this section, we provide empirical evidence of the advantages of our attack scheme via joint sentence and word paraphrasing on both two WCNN and LSTM models and various classification tasks. Our
code for replicating our experiments is available online\footnote{\url{https://github.com/cecilialeiqi/adversarial_text}}.

\subsection{Tasks and Models.}
\label{sec:setting}
We focus on attacking the following state-of-the-art models which also echo our theoretical analysis:
\begin{itemize}
    \item \textbf{Word-level Convolutional Network (WCNN).}\\
We implement a convolutional neural network ~\cite{kim2014convolutional} with a temporal convolutional layer of kernel size 3 and a max-pooling layer, followed by a fully connected layer for the classification output. 
\item \textbf{Long Short Term Memory classifier (LSTM).}\\
The LSTM Classifier~\cite{hochreiter1997long} is well-suited to classifying text sequences of various lengths. We construct a one-layer LSTM with 512 hidden nodes, following the architecture used in ~\cite{kuleshov2018adversarial,DBLP:journals/corr/ZhangZL15}.
\end{itemize}
We carried out experiments on three different text classification tasks: fake-news detection, spam filtering and sentiment analysis; these tasks are also considered in ~\cite{kuleshov2018adversarial}. The corresponding datasets include: 
\begin{itemize}
    \item \textbf{Fake/Real News.}\\
The fake news repository ~\cite{McIntire2017Fake} contains 6336 clean articles of both fake and real news in a 1:1 ratio (5336 training and 1000 testing), with both left- and right-wing sites as sources.
\item \textbf{Trec07p (emails).}\\
The TREC 2007 Public
Spam Corpus (Trec07p) contains 75,419 messages of ham (non-spam) and spam in a 1:2 ratio. We preprocess the data and retain only the main content in each email. We randomly hold out 10\% as testing data.
\item \textbf{Yelp reviews.}\\
The Yelp reviews dataset was obtained from the Yelp Dataset Challenge in 2015. The polarity dataset we used was constructed for a binary classification task that labels 1 star as negative and 5 star as positive. The dataset contains 560,000 training and 38,000 testing documents. 
\end{itemize}

\subsection{General Settings}
For the training procedure, we use similar settings for the WCNN and LSTM classifier. We extracted the top 100,000 most frequent words to form the vocabulary. The first layer of both WCNN and LSTM is the embedding that transforms individual word into a 300-dimensional vector using the pretrained $word2vec$ embeddings ~\cite{mikolov2013efficient}. 
We randomly hold out 10\% training data as validation set to choose the number of epochs and use a constant mini-batch size of 16.  

We manually selected the hyperparameters for each dataset. We set the termination threshold $\tau=0.7$, and set a neighbor size $k$ for possible paraphrases to be 15. We set the semantic similarity $\delta_w=\delta_s=0.75$\footnote{We use the WMD similarity in python's spacy package. The similarity is in [0,1] basis where 1 means identical and 0 means complete irrelevant.} for all datasets and syntactic bound $\delta_2=2$ for news and yelp datasets, and $\delta=\infty$ for Trec07p; the email dataset contains many corrupted words rendering the language model ineffective. For all datasets, we only allow $\lambda_w=20\%$ word paraphrasing. We set the sentence paraphrasing ratio $\lambda_s=20\%$ for yelp and news dataset, and for spam $\lambda_s=60\%$.  

\begin{table*}[htb]
    \centering
    \begin{tabular}{c|c|c|c|c|c|c|c|c}
    \toprule
        \multirow{2}{*}{Dataset} & \multicolumn{4}{c|}{WCNN}  & \multicolumn{4}{|c}{LSTM} \\
        		\hhline{~--------}
        & Origin & ADV (ours) & \multicolumn{2}{c|}{ADV ~\cite{kuleshov2018adversarial}} & Origin & ADV (ours) & \multicolumn{2}{|c}{ ADV~\cite{kuleshov2018adversarial}} \\
        \midrule
        News & 93.1\% & {\bf 35.4\%} & 71.0\% & 70.5\%* & 93.3\% & {\bf 16.5\%} & 37.0\% & 22.8\%*  \\
        \hline
        Trec07p & 99.1\% & {\bf 48.6\%} & 64.5\% & 63.5\%* & 99.7\% & {\bf 31.1\%} & 39.8\% & 37.6\%*  \\
        \hline
        Yelp & 93.6\% & {\bf 23.1\%} & 39.0\% & 41.2\%* & 96.4\% & 30.0\% & {\bf 24.0\%} &  29.2\%* \\
        \bottomrule
    \end{tabular}
\caption{Classifier accuracy on each dataset. Origin and ADV respectively stand for the clean and adversarial testing results. For all datasets, we set word paraphrasing ratio to be $\lambda_w=20\%$ for our method (ADV(ours)). We include results from \cite{kuleshov2018adversarial} for comparison. The first column indicates reported values in their paper; while the consequent column marked by asterisk is our implementation using greedy method in ~\cite{kuleshov2018adversarial} and the same word neighboring set as our method. Both results use large $\lambda_w=50\%$ and allow many more word replacements. }
 \label{tab:accuracy_compare} 
\end{table*}
\subsection{Accuracy comparisons.}
\label{sec:empirical}
After setting up the experimental environment, we now present the empirical studies in several aspects. 
In Table \ref{tab:accuracy_compare} we present the original and adversarial test accuracy on the three datasets with the two chosen models, where we allow $20\%$ word replacements. We also include the presented adversarial accuracy from ~\cite{kuleshov2018adversarial} for reference. Since the word neighboring sets for the two methods are different and the values are not directly comparable, one might argue that we have broaden the search space of words to make the problem easier. Therefore we also implemented the greedy mechanism in \cite{kuleshov2018adversarial} using the same word replacement set as our method has chosen (marked by $*$). Both the reported values from \cite{kuleshov2018adversarial} and our implementation allow $50\%$ word replacements. From Table \ref{tab:accuracy_compare} we can see that in both settings, we are able to successfully flip more prediction classes with fewer word paraphrases. We hereby conclude that joint sentence and word level paraphrasing is much more effective than mere word replacements. Meanwhile, since sentence-level attacks almost perfectly preserve the original meaning, our method can be less susceptible to humans. In the appendix we use some concrete examples to show the significantly improved quality of our generated adversarial texts compared to ~\cite{kuleshov2018adversarial,gong2018adversarial}.\footnote{Since the former code is not available online, we implemented their algorithms. We use their chosen parameters to generate the adversarial examples to compare the quality of sentences in the appendix. While in Section \ref{sec:GGGM} we use the same word neighboring sets for all algorithms to make a fair comparison of the optimization schemes.} In the examples, we can see that sometimes by simplifying or changing the language, or even by making the slightest changes like adding or erasing space, the sentence paraphrase can make a tremendous difference to the classifier output. Consequently, our method does far fewer word level alterations than other methods and greatly reduces the possibility of syntactic or grammar errors.

\begin{figure*}
    \centering
\begin{tabular}{ccc}
    \includegraphics[width=0.34\linewidth,height=0.31\linewidth]{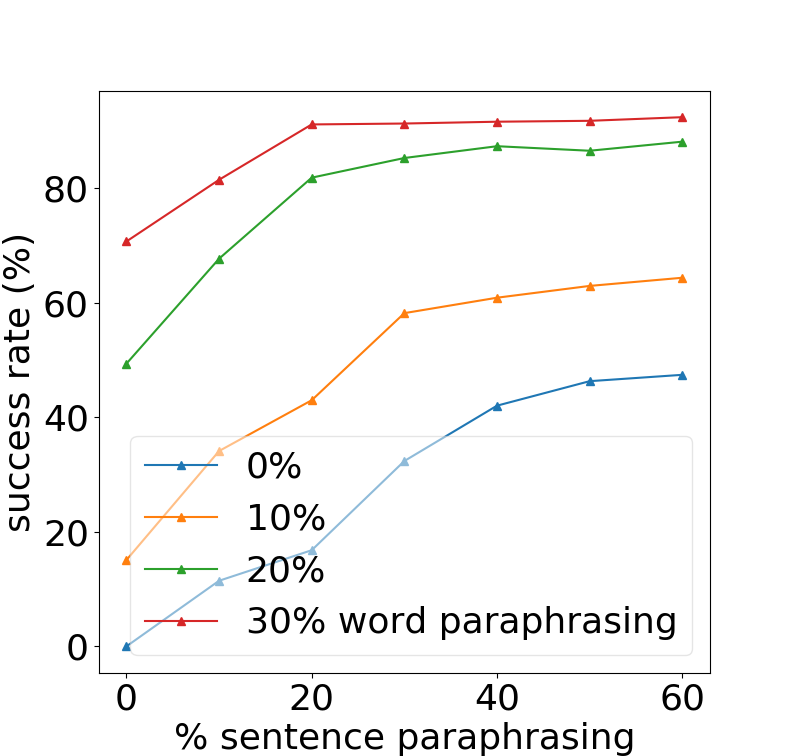} \hspace{-0.8cm}  &     \includegraphics[width=0.34\linewidth,height=0.31\linewidth]{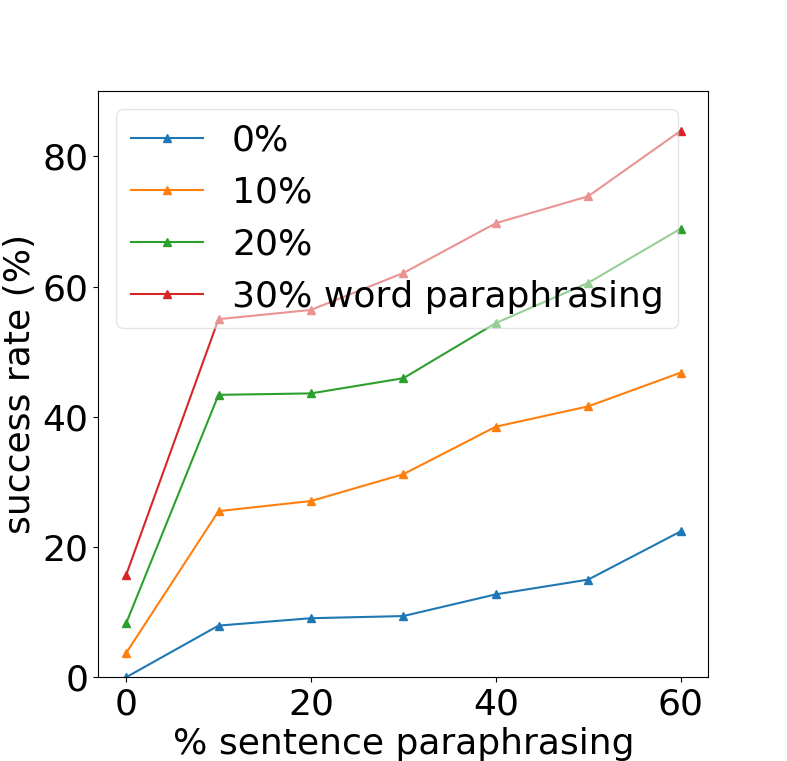} \hspace{-0.8cm}  &     \includegraphics[width=0.34\linewidth,height=0.31\linewidth]{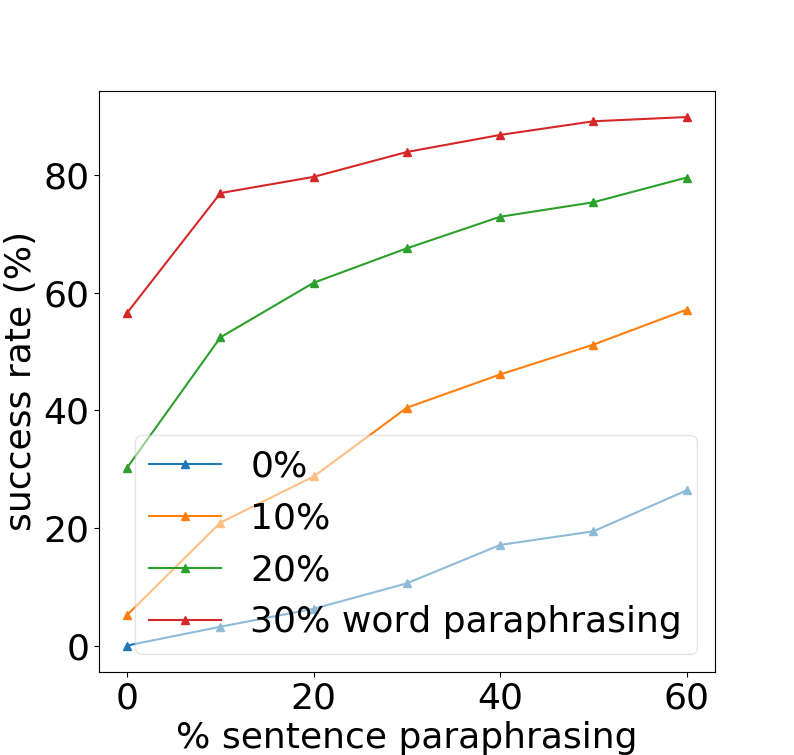}   \\
    News & Trec07p & Yelp
\end{tabular}
        \caption{Success rate of attacking the LSTM classifier with different ratios of allowed paraphrasing.}
    \label{fig:success_rate}
\end{figure*}

To further investigate the joint effect from combining sentence and word level attacks, we also study how each model is susceptible to different degrees of change permitted for both attack levels. 
Therefore we tested and presented the joint influence in Figure \ref{fig:success_rate} for ratios of sentence paraphrasing $\lambda_s$ ranging from 0\% to 60\%, as well as for allowed word paraphrasing percentages $\lambda_w$: 0\%, 10\%, 20\% and 30\%. In all datasets, sentence paraphrasing is especially effective when we allow only a few word paraphrases. For instance, in the sentiment analysis task, we could only successfully attack around 5\% reviews by paraphrasing 10\% of words. But after conducting 60\% sentence paraphrasing beforehand, the success rate increases to almost 60\%.

\begin{table*}[htb]
    \centering
    \begin{tabular}{c|r c|c|c|c|c|c}
    \toprule
\multirow{2}{*}{Method}  & \multicolumn{3}{|c|}{objective-guided greedy \cite{kuleshov2018adversarial} } & \multicolumn{2}{|c|}{gradient method \cite{gong2018adversarial}} & \multicolumn{2}{|c}{ ours (Alg. \ref{alg:word})}\\
\hhline{~-------}
& & $\lambda_w=5\%$ & $\lambda_w=20\%$ & $\lambda_w=5\%$ & $\lambda_w=20\%$ & $\lambda_w=5\% $ & $\lambda_w=20\%$  \\
    \midrule
\multirow{2}{*}{News} & {\em SR:} & 26.2\% & 28.4\% & 9.93\% & 12.8\% & \textbf{39.7\%} & \textbf{45.4\%} \\
& {\em time:} & 0.79 & 1.46 & \textbf{0.13} & \textbf{0.21} & 0.26 & 0.31 \\
\hline
\multirow{2}{*}{Trec07p} &{\em SR:} & 5.1\% & 24.9\% & 0.86\% & 3.4\% & {\bf 12.9 \%} & {\bf 45.3 \%} \\
&{\em time:} & 0.19 & 0.33 & {\bf 0.03} & {\bf 0.05} & 0.07 & 0.09 \\
\hline
\multirow{2}{*}{Yelp} & {\em SR:} & 12.7\% & 45.0\% & 4.2\% & 9.1\% & {\bf 20.7\%} & { \bf 55.9\%} \\
& {\em time:} & 0.15 & 0.21 & {\bf 0.02} & {\bf 0.03} & {\bf 0.02} & 0.05 \\
\bottomrule
    \end{tabular}
    \caption{Attack success rate (denoted by SR) and time comparisons of each optimization mechanism. The performance is reported on the WCNN classifier. Here objective-guided greedy indicates the greedy method used in ~\cite{kuleshov2018adversarial}, and the gradient method is the one suggested in ~\cite{gong2018adversarial}. We can see that even when only applying Algorithm \ref{alg:word}, our optimization method is more effective among others. }
\label{tab:method_compare}
\end{table*}
\subsection{Optimization Method Comparisons for Word-level Attacks.}
To investigate the effectiveness of our proposed gradient-guided greedy method, we implement and compare the time consumption and success rate with Algorithm \ref{alg:word} and the other two techniques: the gradient method ~\cite{gong2018adversarial} and the objective-guided greedy method ~\cite{kuleshov2018adversarial}. To make a fair comparisons of the optimization schemes, we do not conduct sentence level paraphrasing in any of the methods, and we use the same hyperparameters and settings as suggested in Section \ref{sec:setting}. We observe that our scheme is especially more appealing to WCNN, partially because we used 5\% dropout for inference. Recent work~\cite{gal2016dropout} indicates dropout not only works for training but also for inference as a Bayesian approximation.
The small alteration of one word replacement per iteration \cite{kuleshov2018adversarial} is not significant enough to be considered as true gains or the noise from the dropout. While our method replaces 5 words per iteration to capture more difference, thus it is easier to distinguish the change from the dropout randomness. 
From Table \ref{tab:method_compare} we can see that our method requires only $1/5$ to $1/3$ time cost relative to the objective-guided greedy method and also achieves better success rate. On the other hand, gradient method fails to produce good performance when we allow a small set of word replacements. 

\begin{table*}[hbt]
    \centering   
    \begin{tabular}{c|c|c|c|c|c|c}
    \toprule
        \multirow{2}{*}{Dataset} 
        & \multicolumn{3}{c|}{Task I}  
        & \multicolumn{3}{|c}{Task II} \\
        			\hhline{~------}
        & News & Trec07p &Yelp  & News &Trec07p & Yelp \\
        \midrule
        Original & 70.0\% & 80.0\%& 100.0\%& 3.06 $\pm$ 0.67 & 3.23 $\pm$ 0.31 & 1.93 $\pm$ 0.55 \\
        Adversarial & 50.0\% & 80.0\% & 100.0\% & 3.13 $\pm$ 0.50 & 3.10 $\pm$ 0.40 & 2.10 $\pm$ 1.05 \\
        \bottomrule
    \end{tabular}
        \caption{Human-subject validation. Task I measures classification accuracy while Task II the subjective likelihood that each example was crafted by a human (scale from 1 to 5). We used five participants, each shown $n=60$ text examples, half original and half generated using our algorithm. The quality of the generated adversarial text (Task II) is near equal to the original and in fact, slightly higher for the Yelp dataset, but this finding is not necessarily statistically significant. } 
    \label{tab:adv_human_evaluation}
\end{table*}
\subsection{Human Evaluation Validation}
Despite the significantly higher attack proportion of our text examples, our aim is to deliver a message that is faithful to and coherent with the original text. To evaluate the quality of these generated text examples, we presented a number of original and adversarial text pairs (randomly shuffled before the test) to five human evaluators. The evaluators were asked to complete two tasks: I) Assign the correct label to each text sample; II) Rate each text sample with respect the the likelihood that was crafted by a human (scale from 1 to 5). We adopted a majority vote for task I, and averaged the results from five evaluators for task II. As shown in Table \ref{tab:adv_human_evaluation}, we found that human evaluators tend to achieve similar performance for each kind of text in both tasks, indicating that text examples generated via joint sentence and word paraphrasing are indeed coherent and faithful to the original texts in the relevant respects. 

\begin{table*}[htb]
    \centering   
    \begin{tabular}{c|c|c|c|c|c|c}
    \toprule
        \multirow{2}{*}{Dataset} & \multicolumn{3}{c|}{LSTM}  & \multicolumn{3}{|c}{ WCNN} \\
        			\hhline{~------}
        & News & Trec07p &Yelp  & News &Trec07p & Yelp \\
        \midrule
        Test (before) & 93.3\% & 99.7\%& 96.4\%& 93.1\%& 99.1\%& 93.6\%\\
        Test (after) & 94.5\% & 99.5\% & 97.3\% & 93.8\% & 99.2\% & 94.9\% \\
        \midrule
     ADV (before) & 16.5\% & 31.1\% & 30.0\% & 35.4\% & 48.6\% & 23.1\% \\
     ADV (after) & 32.7\% & 50.1\% & 46.7\% & 40.0\% & 54.2\% & 44.4\% \\
        \bottomrule
    \end{tabular}
        \caption{Performance of adversarial training. } 
        \label{tab:adv_train}    
\end{table*}

\subsection{Adversarial Training.}
Finally, we investigated whether our adversarial examples can help improve model robustness. For each dataset, we randomly selected 20\% of the training data and generated adversarial examples from them using Algorithm \ref{alg:joint}. We then merged these adversarial examples with corrected labels into the training set and retrained the model. We present the testing and adversarial accuracy before and after this adversarial training process in Table \ref{tab:adv_train}. Under almost all circumstances, adversarial training improved the generalization of the model and made it less susceptible to attack. 

\section{Conclusion}
In this paper, we propose a general framework for discrete adversarial attacks. Mathematically, we formulate the adversarial attack as an optimization task on a set of attacks. We then theoretically prove that greedy method guarantees a $1-1/e$ approximation factor for two classes of neural network for text classification task. Empirically, we propose a gradient-guided greedy method that inherits the efficiency of gradient method and ability to attack of greedy method. Specifically, we investigate joint sentence and word paraphrasing to generate attacking space that maintain the original semantics and syntax for text adversarial examples.

\paragraph{Acknowledgements.} I.D. acknowledges the support of NSF via
IIS-1546452 and CCF-1564000. A.D. acknowledges the
support of NSF Grants 1618689, DMS 1723052, CCF 1763702, ARO YIP W911NF-14-1-0258 and research gifts by Google, Western Digital and NVIDIA.



%% file: appendix_arxiv.tex
\appendix

\section{Proofs}
\label{sec:proof}
\subsection{Proof of Proposition \ref{remark:subsetsum} }
\begin{proof}[Proof of Proposition \ref{remark:subsetsum}]
	
	We will show that even for a very simple function $f$, it could be reduced from subset sum problem when $k\geq 2$. 
	
	For instance, let $$f(S)=\argmax_{\supp(\bl)\subset S } \left\|\sum_{i=1}^n V(x_i^{(l_i)}) -\bv \right\|_2^2,$$ where the target is to find the best $\ell_2$ approximation of some target vector $\bv$ from the embedding vectors. 
	
	For simplicity, denote the embedding vector of each paraphrased words to be $V(x_i^{(j)})=\bv_i^{(j)}, 1\leq i\leq n, 0\leq j\leq k-1$.
	Suppose there is an algorithm that solves the above problem in time polynomial to $n$. Then we will now show that the subset sum problem has a solution in polynomial time. Let the $n$ numbers to be $s_1,s_2,\cdots s_n$, and the target to be $W$. Then we let $\bv_i^{(0)}=[s_i,0,0,\cdots,0]$, and $\bv_i^{(j)}=\mathbf{0},j=1,\cdots k-1$, with target $\bv=[W,0,0,\cdots,0]$. Then just check if the best approximation of $\bv$ is exactly $\bv$ will suffice the subset sum problem. Therefore it contradicts with the fact that subset sum is in NP-complete class.
\end{proof}

\subsection{Proof of Proposition \ref{lemma:frank_wolf_P} }
\begin{proof}[Proof of Proposition \ref{lemma:frank_wolf_P}]
Define set function $h(S)= \argmax_{\supp(\bl)\subset S} V(T_{\bl}(\bx))^\top \nabla C_{y}(\bv)$, where $\bv=V(\bx)$. Denote $\bg=\nabla C_{y}(\bv)$. 
	When $V$ is bag-of-words embedding, we denote the embedding of each paraphrased word in $\bx=[x_1,x_2,\cdots,x_n]$ as $V(x_i^{(j)}) =\be_{d_{ij}}$. Here for any $i$, $\be_i$ is defined as the one-hot vector with 1 in index $i$ and 0 elsewhere.
	Then  $V(T_{\bl}(\bx))=\sum_{i=1}^n \be_{d_{il_i}}$.
	
	\begin{align*}
	h(S) &= \argmax_{\supp(\bl)\subset S} V(T_{\bl}(\bx))^\top \bg    \\
	&= \argmax_{\supp(\bl)\subset S} \sum_{i=1}^n \be_{d_{il_i}}^\top \bg \\
	&= \argmax_{\supp(\bl)\subset S} \sum_{i=1}^n g_{d_{il_i}}\\
	&= \mathbb{1}_{S}^\top \bw,  
	\end{align*}
	where $w_i=\max_{0\leq t\leq k-1} g_{d_{it}} $. 
	
	When $V$ is $d$-dimentional {\em word2vec} embedding, the embedding $V(\bx)=[V(x_1)^\top |V(x_2)^\top |\cdots|V(x_n)^\top ]^\top \in \R^{nd}$. Denote $\hat{\bg}_i=\bg_{(id-d+1):id}$ to be the gradient with respect to the word $w_i$.
	
	\begin{align*}
	h(S) &= \argmax_{\supp(\bl)\subset S} V(T_{\bl}(\bx))^\top \bg    \\
	&= \argmax_{\supp(\bl)\subset S} \sum_{i=1}^n V(x_i^{(l_i)})^\top \hat{\bg}_i \\
	&= \mathbb{1}_{S}^\top \bw,  
	\end{align*}
	where $w_i=\max_{0\leq t\leq k-1} V(x_i^{(t)})^\top \hat{\bg}_i $. Therefore for both bag-of-words embedding and {\em word2vec} embedding, $h$ is a modular (linear) set function, and Problem \ref{problem:frank-wolfe} is solvable in polynomial time. 	
\end{proof}

\subsection{Proof of Claim \ref{lemma:monotone} }
\begin{proof}[Proof of Claim \ref{lemma:monotone}]
	Clearly for any $S\subset V\subset [n]$,
	\begin{align*}
	f(S)&= \max_{\supp(\bl)\subset S} C_{y}(V(T_{\bl}(\bx))) \leq \max_{\supp(\bl)\subset V} C_{y}(V(T_{\bl}(\bx))) \shortrnote{(since $S\subset T$)}\\
	&= f(V)
	\end{align*}
	Therefore the set function $f$ is non-decreasing. Since the problem of maximizing a monotone submodular function subject to a cardinality constraint admits a $1-1/e$ approximation algorithm\cite{nemhauser1978analysis}, Problem \ref{problem:origin} can be solved in time polynomial to $n$ with greedy method.
\end{proof}

\subsection{Proof of Theorem \ref{remark:cnn_submodular} }
\begin{proof}[Proof of Theorem \ref{remark:cnn_submodular}]
	We start from a simple case, $h=1$, i.e., a unit kernel size, and we look at a single feature corresponding to one filter, i.e. $\hat{c}_j=\max_{i=1}^n c_{ij}$. 
	
	To further incorporate the transformation to the input, we rewrite $\hat{c}_j$ as a function of the transformation index $\bl$.
	\begin{eqnarray*}
		\hat{c}_j(\bl)&\equiv& \max_{i=1}^{n} \phi( \bw_j^\top V(x_i^{(l_i)}) + b_j)= \max_{i=1}^n v_{ij}^{(l_i)},
	\end{eqnarray*}
	where $\bw_j$ is the $j$-th filter and we denote $v_{ij}^{(k)} = \phi(\bw_j^\top V(x_i^{(k)}) + b_j)$ for simplicity. 
	
	Let $S, T$ denote two sets that satisfy $S\subset T \subset [n]$.
	For any two vectors $\bl^S$ and $\bl^T$ satisfy that 
	$l^S_i=l^T_i, \forall i\in S$, and $\supp(\bl^S)=S, \supp(\bl^T)=T$. 
	With the assumption that $\bw_j^\top V(x_i)\leq \bw_j^\top V(x_i^{(t)})$, and since the activation function is non-decreasing, we have $v_{ij}^{(0)}\leq v_{ij}^{(t)}, \forall i\in [n],j\in[m], t\in [k-1]$, and hereby $\hat{c}_j(\bl^S)\leq \hat{c}_j(\bl^T)$.
	
	Therefore for any new element's position $s$ and its replacement index $t$, we have 
	\begin{align*} 
	&\hat{c}_j(\bl^S+t\be_s) - \hat{c}_j(\bl^S) = \max\{ v_{sj}^{(t)} - \hat{c}_j(\bl^S),0\} \\
	\geq& \max\{ v_{sj}^{(t)} - \hat{c}_j(\bl^T),0\} \shortrnote{(since $\hat{c}_j(\bl^S)\leq \hat{c}_j(\bl^T)$)} \\
	=& \hat{c}_j(\bl^T+t\be_s)- \hat{c}_j(\bl^T).
	\end{align*}
	
	Since the final output probability is a positive weighted summation of each $\hat{c}_j$, it also satisfies
	\begin{eqnarray}
	\nonumber
	&& C^{\text{WCNN}}(\bl^S+t\be_s) -C^{\text{WCNN}}(\bl^S)	 \geq  C^{\text{WCNN}}(\bl^T+t\be_s) -C^{\text{WCNN}}(\bl^T)
	\label{eqn:cnnST}
	\end{eqnarray} 
	Taking the max over all $\bl^S$, $\bl^T$ we have:
	$$f(S+\{s\}) =  \max_{t=1}^{k-1}\max_{\supp(\bl^S)=S}  C^{\text{WCNN}}(\bl^S+t\be_s)$$
	Therefore
	\begin{align*}
		&f(S+\{s\})-f(S)\\
		=& \max_{t=1}^{k-1}\left\{ \max_{\supp(\bl^S)=S}  \left\{C^{\text{WCNN}}(\bl^S+te_s) - C^{\text{WCNN}}(\bl^S)\right\} \right\}\\
		\geq &\max_{t=1}^{k-1} \left\{ \max_{\supp(\bl^T)=T} \left\{ C^{\text{WCNN}}(\bl^T+te_s) - C^{\text{WCNN}}(\bl^T)\right\} \right\} \shortrnote{(from \eqref{eqn:cnnST})}  \\
		=& f(T+\{s\}) -f(T).
	\end{align*}
	The case when $2\leq h\leq s$ is essentially the same with $h=1$ since each window has no overlapping. We could simply replace $\bv_1$ by $\bv_{1:h}$ and conduct the same analysis. 
\end{proof}

\subsection{Proof of Theorem \ref{remark:rnn_submodular}}
\begin{proof}[Proof of Theorem \ref{remark:rnn_submodular}]
	Recall that the hidden state node $h_i$ is defined recursively as: 
	\begin{align*}
h_0 &= C, \shortrnote{($C$ is constant)} \\
h_i &= \phi(wh_{i-1}+ \bm^\top V(x_{i-1})+b). 
\end{align*}
	And the classifier output is $C^{\text{RNN}}(V(\bx))=yh_T$.
	
	For simplicity, we denote $v_i^{(j)} \equiv \bm^\top V(x_{i}^{(j)})+b$. Since we will only look for the transformation that maximizes the classifier output, without loss of generality, we assume $v_i^{(j)}\geq v_i^{(0)}, \forall i\in [T], j\in [k-1]$.
	
	For a fixed input $\bx=[x_1,x_2,\cdots,x_T]$ and transformation index $\bl$, we want to study how changing an intermediate hidden state affects the consecutive layers' output. Therefore we represent the value of a $j$-th hidden state as a function of the $i$-th hidden node and the transformation label $\bl$, that captures the network from $i$-th through $j$-th time steps, i.e.,
	$$ f_{i:j}(h_i,\bl) = \phi\left(w\cdots \phi(wh_i+v_i^{(l_i)})+\cdots +v_{j-1}^{(l_{j-1})}\right). $$
	
	Finally we want to study the whole network's output $yf_{0:T}(C,\bl)$.
	We first prove the following lemma:
	\begin{lemma}
		\label{lemma:support}
		\begin{eqnarray}
		&&f_{i:j}(h_i+\delta,\bl)-f_{i:j}(h_i,\bl)\geq f_{i:j}(h_i+\delta,\bl+t\be_s)-f_{i:j}(h_i,\bl+t\be_s),
		\label{eqn:support}
		\end{eqnarray}
		for any $0\leq i<j\leq T, t\in [k-1], s\in [T], s\notin \supp(\bl), \delta>0$. 
	\end{lemma}
	\begin{proof}[Proof of Lemma \ref{lemma:support}]
		\begin{eqnarray*}
			&&f_{i:j}(h_i+\delta,\bl+t\be_s)\\
			&=&f_{s+1:j}(\phi(wf_{i:s}(h_i+\delta,\bl+t\be_s)+v_s^{(t)}),\bl+t\be_s)\\
			&=&f_{s+1:j}(\phi(wf_{i:s}(h_i+\delta,\bl)+v_s^{(t)}),\bl)
		\end{eqnarray*}
		Now we simplify the equation by define $a(\delta,t)=\phi(wf_{i:s}(h_i+\delta,\bl)+v_s^{(t)}), \delta \in \R, t\in [k-1]$. Therefore we could rewrite the four terms in Eqn. \eqref{eqn:support} as:
		$f_{s+1:j}(a(\delta,0),\bl)$,
		$f_{s+1:j}(a(0,0),\bl)$,
		$f_{s+1:j}(a(\delta,t),\bl)$,
		and $f_{s+1:j}(a(0,t),\bl)$.
		
		Since $\phi$ is concave and $v^{(t)}_s\geq v^{(0)}_s$, notice 
		\begin{equation}
		a(\delta,t)-a(0,t)\leq a(\delta,0)-a(0,0).   \label{eqn:a}
		\end{equation}
		Now since $f_{s+1:j}(\cdot,\bl)$ is a composite of concave function and is also concave, we have:
		\begin{align}
		\nonumber
		&f_{i:j}(h_i+\delta,\bl+t\be_s)-f_{i:s}(h_i,\bl+t\be_s)\\
		\nonumber
		=&f_{s+1:j}(a(\delta,t),\bl)-f_{s+1:j}(a(0,t),\bl)\\
		\nonumber
		\leq& f_{s+1:j}(a(\delta,0)+a(0,t)-a(0,0),\bl)-f_{s+1:j}(a(0,t),\bl) \longrnote{(from \eqref{eqn:a} and non-decreasing $f_{s+1:j}(\cdot,\bl)$)}\\
		\nonumber
		=&f_{s+1:j}(a(\delta,0)+(a(0,t)-a(0,0)),\bl)-f_{s+1:j}(a(0,0)+ (a(0,t)-a(0,0)),\bl)\\
		\nonumber
		\leq& f_{s+1:j}(a(\delta,0),\bl)-f_{s+1:j}(a(0,0),\bl) \shortrnote{(from concavity of $f_{s+1:j}(\cdot,\bl)$)}\\
		=& f_{i:j}(h_i+\delta,\bl)-f_{i:s}(h_i,\bl)
		\label{eqn:insidelemma}
		\end{align}
		\end{proof}
		Lemma \ref{lemma:support} could be extended to a more general form. Suppose two indices $\bl^S$ and $\bl^U$ satisfy $\supp(\bl^S)=S, \supp(\bl^U)=U, S\subset U$, and $l^S_i=l^U_i,\forall i\in S$. Since we could write $\bl^U$ as $\bl^S+\sum_{i\in U\backslash S} l^U_i\be_i$, by repeatedly using Lemma \ref{lemma:support} we have:
		\begin{eqnarray}
		\nonumber
		&&f_{i:j}(h_i+\delta,\bl^S)-f_{i:j}(h_i,\bl^S)
		\geq f_{i:j}(h_i+\delta,\bl^U)-f_{i:j}(h_i,\bl^U).
		\label{eqn:support2}
		\end{eqnarray}
		This conclusion basically claims an increase into an intermediate layer of the network will have smaller effect to the output when the network is attacked on more time steps. Then back to Theorem \ref{remark:rnn_submodular}. Now consider we add a coordinate $s$ to the set $S$ and $U$, $s\notin \supp(S)\cup \supp(U)$.
		\begin{align}
		\nonumber
		& f_{0:T}(C,\bl^S+t\be_s)-f_{0:T}(C,\bl^S)    \\
		\nonumber
		= &f_{s:T}(\phi(w f_{0:s-1}(C,\bl^S)+v_s^{(t)}),\bl^S)-f_{s:T}(\phi(w f_{0:s-1}(C,\bl^S)+v_s^{(0)}),\bl^S)\\
		\nonumber
		\geq& f_{s:T}(\phi(w f_{0:s-1}(C,\bl^S)+v_s^{(t)}),\bl^U) -f_{s:T}(\phi(w f_{0:s-1}(C,\bl^S)+v_s^{(0)}),\bl^U) \longrnote{(from \eqref{eqn:support2} and since $\phi$ is non-decreasing, $v_s^{(t)}\geq v_s^{(0)}$)}\\
		\nonumber
		\geq& f_{s:T}(\phi(w f_{0:s-1}(C,\bl^U)+v_s^{(t)}),\bl^U)-f_{s:T}(\phi(w f_{0:s-1}(C,\bl^U)+v_s^{(0)}),\bl^U) \longrnote{(since $f_{s:T}(\cdot,\bl^U)$ is concave and similar analysis as \eqref{eqn:insidelemma})}\\
		=& f_{0:T}(C,\bl^U+t\be_s)-f_{0:T}(C,\bl^U)
		\label{eqn:proof_with_l}
		\end{align}
		Finally, since
		\begin{align*}
		&\max_{\supp(\bl)\subset S\cup\{s\}} C^{\text{RNN}}(V(T_{\bl}(\bx)))=\max_{\supp(\bl^S)\subset S}\max_{t\in[k-1]} yf_{0:T}(C,\bl^S+t\be_s), 
		\end{align*}
		we have:
		\begin{align*}
		&\max_{\supp(\bl)\subset S\cup\{s\}} C^{\text{RNN}}(V(T_{\bl}(\bx)))- \max_{\supp(\bl)\subset S} C^{\text{RNN}}(V(T_{\bl}(\bx)))\\
		=&\max_{\supp(\bl^S)\subset S}(\max_{t\in[k-1]} yf_{0:T}(C,\bl^S+t\be_s) - yf_{0:T}(C,\bl^S))\\
		\geq &\max_{\supp(\bl^U)\subset U}(\max_{t\in[k-1]} yf_{0:T}(C,\bl^U+t\be_s) - yf_{0:T}(C,\bl^U)) \shortrnote{(since \eqref{eqn:proof_with_l} holds for any $t$)} \\
		=& \max_{\supp(\bl)\subset U\cup\{s\}} C^{\text{RNN}}(V(T_{\bl}(\bx)))\\
		-& \max_{\supp(\bl)\subset U} C^{\text{RNN}}(V(T_{\bl}(\bx)))
		\end{align*}
	\end{proof}
	
	\section{Data statistics}
	\begin{table}[h]
		\centering
		\begin{tabular}{c|c|c|c}
			\toprule
			Dataset & Task & \#Train & \#Test\\
			\midrule
			Trec07p & Spam filtering &67.9k &7.5k\\
			Yelp & Sentiment analysis &560k &38k\\
			News &Fake news detection &5.3k &1.0k\\
			\bottomrule
		\end{tabular}
		\caption{Statistics of each datasets}
		\label{tab:statistics}
	\end{table}

	\section{Comparisons with other methods with concrete examples}
	\label{appendix:example}
	In this section, we provide some concrete examples to compare our method with the other related methods. The following six examples respectively show the combinations of three datasets (fake news, Trec07p, and yelp) as well as the two models we use (LSTM and WCNN). 
	
	We use red font to denote changes from sentence level paraphrasing and blue for word paraphrasing. 
	
	\subsection{Empirical example 1: \textbf{Task} - Fake news detection. \textbf{Classifier}-CNN.} 
	\textbf{Method}: Ours. \textbf{Origin}: 100\% Real. \textbf{ADV}: 71\% Fake\\
	\st{6} \Red{Six detainees} detained in raids in Belgium \st{Brussels}, Belgium (CNN) Police detained six people in \st{raids} \Red{raid} Thursday night \st{as} \Red{when} investigators \st{raced} \Red{were sent} to uncover the network behind this week's terror attacks in the Belgian capital.  The Belgian federal prosecutor's office didn't provide details about who had been detained in the Brussels raids, why they had been apprehended or whether they will face charges.  It will be decided tomorrow if these people will remain in custody, the office said in a statement released late Thursday.  Two people were taken into custody in Brussels' Jette neighborhood, one person was detained in a different part of the capital, and three people were in a vehicle in front of the federal prosecutor's office when authorities apprehended them, public broadcaster RTBF reported.  \st{So far, authorities have} \Red{Authorities} said they believe five \st{men played a part} \Red{people took a shot} in Tuesday's bombings in Belgium that killed 31 people and \st{injured} \Red{wounded} 330. Three of the attackers are dead. Two of them could still be on the loose.  Investigators are combing over evidence from surveillance footage and the explosives stash they seized from an apparent hideaway in a suburb.  Sweeps where investigators detain people first and ask questions later are likely to become an increasingly common tactic, CNN national security analyst Juliette Kayyem said.  There will be lots more of them, she said. They are going to be what's called overbroad. They are going to just try to find people or evidence that may stop the next terrorism attack, and they will figure out who they have under custody.  Khalid El Bakraoui, one of the terrorists who bombed a train near the Maelbeek metro station, is dead. Authorities believe a second unidentified person was also involved in that attack, a senior Belgian security source told CNN. But investigators don't know where \st{that} \Red{the} suspect is -- or whether \st{he's} \Red{he was} dead or alive.  Surveillance footage shows the man holding a large bag at the station, according to Belgian public broadcaster RTBF. It's not clear if he was among the at least 20 killed in that blast, RTBF said.  Authorities have released a grainy image of another suspect who they believe is on the run.  That man, they say, shown in photographs wearing a black hat, was one of three attackers at Brussels Airport. Authorities say he planted a bomb at the airport and left. The other two men in the photographs are believed to be the suicide bombers.  Fair to ask whether 'we missed the chance'  Did Belgian authorities miss a chance to stop at least one of the suspects involved in the attacks?  Bakraoui had been sentenced to nine years in prison in Belgium back in 2010 for opening fire on police officers with a Kalashnikov during a robbery, according to broadcaster RTBF and CNN affiliate RTL. Needless to say, he didn't serve all that time.  Given the facts, it is justified that ... people ask how it is possible that someone was released early and we missed the chance when he was in Turkey to detain him, said Jambon, whose offer to resign was rebuffed by Prime Minister Charles Michel.  Investigators suspect Abdeslam planned to be part of an attack by the same ISIS cell that lashed out Tuesday, a senior Belgian counterterrorism official told CNN's Paul Cruickshank.  Authorities looked Wednesday at the Brussels homes of the Bakraoui brothers. \st{Those} \Red{These} two \st{searches} \Red{findings} were not \st{conclusive} \Red{decisive}, the federal \st{prosecutor's office} \Red{prosecutors} said.  Homes were searched Thursday in several areas in and around the city, officials said.  One operation in the neighborhood of Schaerbeek stretched for hours into Friday morning. Investigators sealed off streets for several blocks. It was not immediately clear why such a large area had been cordoned.  Masked teams in hazmat gear could be seen exiting a building and heading toward a police van.  As investigations continue, a larger question looms: What could happen next?  Not long ago, Western authorities believed ISIS was focused on taking territory in Syria and Iraq, not lashing out elsewhere. But U.S. officials now think the extremist group has been sending trained militants to Europe for some time.  These men don't necessarily follow orders directly from ISIS headquarters. But they build on what they've learned, as well as a shared philosophy and approach, to develop their own terror cells and hatch their own plots.  How many more ISIS militants are in Europe, poised to attack? That's not clear.  For now, though, the top priority is tracking down the two men linked directly to Tuesday's terror.
	
	\textbf{Method}: Greedy\cite{kuleshov2018adversarial}. \textbf{Origin}: 100\% Real. \textbf{ADV}: 79\% Fake\\
	\st{6} \Blue{7} \st{detained} \Blue{detention} in raids in Belgium Brussels, Belgium (CNN) \st{Police} \Blue{cops} \st{detained} \Blue{deported} six people in raids Thursday night as \st{investigators} \Blue{investigation} raced to uncover the network behind this week's \st{terror} \Blue{terrorists} attacks in the Belgian capital. The Belgian federal prosecutor's office didn't provide details about who had been detained in the Brussels raids, why they had been apprehended or whether \st{they} \Blue{we} \st{will} \Blue{should} \st{face} \Blue{eyes} charges. It will be decided tomorrow if these people will remain in custody, the office \st{said} \Blue{told} in a \st{statement} \Blue{stating} released late Thursday. Two people were taken into custody in Brussels' Jette neighborhood, one person was \st{detained} \Blue{detention} in a different part of the capital, and three people were in a vehicle in front of the federal prosecutor's office when authorities apprehended them, public broadcaster RTBF reported. So far, authorities have said they believe five men played a part in Tuesday's bombings in Belgium what \st{killed} \Blue{wounded} \st{31} \Blue{26} \st{people} \Blue{individuals} and injured 330. Three of the attackers are dead. Two of them could still be on of loose. \st{Investigators} \Blue{Investigating} \st{are} \Blue{they} combing over evidence from surveillance \st{footage} \Blue{filmed} and the explosives stash \st{they} \Blue{never} seized from an \st{apparent} \Blue{obvious} hideaway in a suburb. Sweeps where investigators detain people first and ask questions later are likely to become an increasingly \st{common} \Blue{commonly} tactic, CNN national Security \st{analyst} \Blue{analysts} Juliette Kayyem \st{said} \Blue{say}. There will be lots more of them, \st{she} \Blue{knew} \st{said} \Blue{say}. They are going to be what's called overbroad. They are going to just try to find people \st{or} \Blue{and/or} \st{evidence} \Blue{findings} that may stop \st{the} \Blue{of} \st{next} \Blue{before} Terrorism attack, and they will figure out who they have under custody. Khalid El Bakraoui, one of the terrorists who bombed a train near the Maelbeek metro station, is dead. Authorities believe a second unidentified person was also involved in that attack, a senior Belgian Security \st{source} \Blue{sources} \st{told} \Blue{saying} CNN. But investigation don't know where that \st{suspect} \Blue{victim} \st{is} \Blue{be} -- \st{or} \Blue{any} \st{whether} \Blue{not} he's \st{dead} \Blue{dying} \st{or} \Blue{and/or} alive. Surveillance footage shows the man holding a large bag at the station, according to Belgian public broadcaster RTBF. It's not clear if he was among the at least \st{20} \Blue{26} \st{killed} \Blue{kill} in that blast, RTBF \st{said} \Blue{say}. Authorities have released \st{a} \Blue{another} grainy image \st{of} \Blue{the} another suspect who they believe is on the run. That man, they say, shown in \st{photographs} \Blue{photo} \st{wearing} \Blue{wear} \st{a} \Blue{one} \st{black} \Blue{red} hat, was one of \st{three} \Blue{twelve} \st{attackers} \Blue{attacker} at Brussels Airport. Authorities say he planted a bomb at the airport and left. The other two men in the photographs are \st{believed} \Blue{supposedly} to be the \st{suicide} \Blue{suicidal} bombers. Fair to ask whether 'we missed the chance' Did Belgian authorities miss a chance to stop at least one of the suspects involved in the attacks? Bakraoui had been sentenced to nine years in prison in Belgium back in 2010 for opening fire on \st{police} \Blue{policemen} \st{officers} \Blue{deputies} with a Kalashnikov during a robbery, according to broadcaster RTBF and CNN affiliate\Blue{s} RTL. Needless to say, he didn't serve all that time. Given the facts, it is justified \st{that} \Blue{actually} ... \st{people} \Blue{everyone} \st{ask} \Blue{tell} how it is possible \st{that} \Blue{what} someone was released early and \st{we} \Blue{you} missed of chance when he was in Turkey to detain him, \st{said} \Blue{told} Jambon, who\st{se} offer to resign was rebuffed by Prime Minister Charles Michel. \st{Investigators} \Blue{Investigation} suspect\Blue{ed} Abdeslam \st{planned} \Blue{planning} to be part of an \st{attack} \Blue{enemy} by the \st{same} \Blue{different} ISIS cell\Blue{s} that lashed out Tuesday, a \st{senior} \Blue{junior} Belgian counter-terrorism \Blue{un}official told CNN's Paul Cruickshank. Authorities looked Wednesday at the Brussels homes of the Bakraoui brothers. \st{Those} \Blue{Them} \st{two} \Blue{one} searches were not conclusive, the federal prosecutor's Office \st{said} \Blue{say}. \st{Homes} \Blue{Houses} were searched Thursday in several areas in and around the city, \st{officials} \Blue{authorities} \st{said} \Blue{say}. One operation in the neighborhood of Schaerbeek stretched for hours into Friday morning. Investigators sealed off streets for several blocks. It was not immediately clear why such a large area had been cordoned. Masked teams in hazmat gear could be seen exiting a building and heading toward a police van. As investigations continue, a larger question looms: What could happen next? Not long ago, Western authorities believed ISIS was focused on taking territory in Syria and Iraq, not lashing out elsewhere. But U.S. officials now think the extremist group has been sending trained militants to Europe for some time. These men don't necessarily follow orders directly from ISIS headquarters. But they build on what they've learned, as well as a shared philosophy and approach, to develop their own terror cells and hatch their own plots. How many more ISIS militants are in Europe, poised to attack? That's not clear. For now, though, the top priority is tracking down the two men linked directly to Tuesday's terror.
	
	\textbf{Method}: Gradient method\cite{gong2018adversarial}. \textbf{Origin}: 100\% Real. \textbf{ADV}: $99.5\%$ Real\\
	
	6 detained in raids in Belgium Brussels, Belgium (CNN) Police detained six people in raids Thursday night \st{as} \Blue{well} investigators \st{raced} \Blue{rode} to uncover \st{the} \Blue{of} \st{network} \Blue{networks} behind \st{this} \Blue{it} week's terror attacks in \st{the} \Blue{of} Belgian capital. The Belgian federal prosecutor's office didn't provide details about who had been \st{detained} \Blue{arrested} in \st{the} \Blue{of} Brussels raids, why \st{they} \Blue{have} had \st{been} \Blue{being} apprehended or whether they \st{will} \Blue{be} face charges. It will should decided tomorrow if these people \st{will} \Blue{be} \st{remain} \Blue{remains} in custody, the \st{office} \Blue{offices} said in a statement released late Thursday. \st{Two} \Blue{Three} people were taken into custody in Brussels' Jette neighborhood, \st{one} \Blue{another} person was detained in a different part \st{of} \Blue{the}  the capital, and three people \st{were} \Blue{had} in a vehicle in front of \st{the} \Blue{of} federal prosecutor's office when \st{authorities} \Blue{officials} apprehended them, public broadcaster RTBF reported. So far, authorities have said they believe five men \st{played} \Blue{playing} \st{a} \Blue{another} part in Tuesday's bombings in Belgium that killed \st{31} \Blue{29} people and injured 330. Three of the attacker be dead. Two of them could still be on the loose. Investigators are combing over evidence from surveillance \st{footage} \Blue{filmed} and the explosives stash they seized from an apparent hideaway in \st{a} \Blue{another} suburb. Sweeps where investigator detain people \st{first} \Blue{second} and ask questions later are likely to become an increasingly common\Blue{ly} tactic, CNN national security analyst Juliette Kayyem said. There will be lot more of them, she \st{said} \Blue{guess}. They are going \st{to} \Blue{able} \st{be} \Blue{should} what's \st{called} \Blue{known} overbroad. They are go\st{ing} \st{to} \Blue{able} just try\Blue{ing} to find people or evidence that may stop the next Terrorism attack, and they will figure\Blue{s} \st{out} \Blue{up} who\Blue{m} they have under custody. Khalid El Bakraoui, one of the terrorists who bombed a train near the Maelbeek metro station, is dead. \st{Authorities} \Blue{Officials} believe a second unidentified person was also involved in that attack, a senior Belgian Security sources \st{told} \Blue{talked} CNN. \st{Blut} \Blue{Though} investigator\st{s} don't \st{know} \Blue{think} \st{where} \Blue{there} \st{that} \Blue{because} suspect is -- or \st{whether} \Blue{if} he's \st{dead} \Blue{dying} \st{or} \Blue{ either} alive. Surveillance footage shows the man holding a large bag at the station, according to Belgian public broadcaster RTBF. It's not clear if he was among \st{the} \Blue{of} at least \st{20} \Blue{25} killed in \st{that} \Blue{because} blast, RTBF \st{said} \Blue{guess}. Authorities \st{have} {\color{blue} 've} release\st{d} a\Blue{nother} grainy image of \st{another} \Blue{one} suspect who \st{they} \Blue{have} believe is on the run. That man, they say, shown in photograph\st{s} wear\st{ing} a \st{black} \Blue{grey} hat, was \st{one} \Blue{another} of three attacker\st{s} at Brussels Airport. Authorities say he planted a bomb at the airport and left. The other two \Blue{wo}men in of photographs are believed to be the suicidal bombers. Fair \st{to} \Blue{able} \st{ask} \Blue{tell} whether 'we missed the chance' Did Belgian authorities miss a chance to stop\Blue{ping} at least one of the suspects involved in of attacks? Bakraoui \st{had} \Blue{came} \st{been} \Blue{being} \st{sentenced} \Blue{convicted} to nine years in prison in \st{Belgium} \Blue{Netherlands} back in 2010 for \st{opening} \Blue{closing} fire\Blue{s} on police\Blue{men} \st{officers} \Blue{policemen} with a\Blue{nother} Kalashnikov during a robbery, according to broadcaster RTBF and CNN affiliate RTL. Needless to say, he didn't serve all that time. Given the facts, it is justified that ... people \st{ask} \Blue{tell} \st{how} \Blue{what} it is possible that someone was released early and we missed of chance when He was in Turkey to detain\Blue{ing} him, said Jambon, whose offer\Blue{ing} to resign was rebuffed by Prime Minister Charles Michel. Investigators suspect Abdeslam planned \st{to} \Blue{able} \st{be} \Blue{should} part \st{of} \Blue{the} an attack by the same ISIS cell that lashed out Tuesday, a junior Belgian counter-terrorism unofficial told CNN's Paul Cruickshank. Authorities \st{looked} \Blue{seemed} Wednesday at the Brussels homes of the Bakraoui brothers. Those two searches were not conclusive, the federal prosecutor's office said. Homes were searched Thursday in \st{several} \Blue{numerous} areas in and around the city, officials \st{said} \Blue{guess}. One operation in \st{the} \Blue{of} neighborhood of Schaerbeek stretched for hours into Friday morning. Investigators sealed off streets for several blocks. It was not immediately clear why such a large areas had \st{been} \Blue{being} cordoned. Masked teams in hazmat gears could be seen exiting a\Blue{nother} building\Blue{s} and heading toward\Blue{s} a\Blue{nother} police van. As investigations continue, a \st{larger} \Blue{sized} question\Blue{s} looms: What could happen next? Not long ago, Western authorities believed ISIS was focused on taking territory in Syria and Iraq, not lashing out elsewhere. But U.S. \st{officials} \Blue{authorities} now \st{think} \Blue{know} the extremist group has being sending trained \st{militants} \Blue{insurgents} to Europe for some time. These men don't necessarily follow orders directly from ISIS headquarters. But they build on what they've learned, as well as a shared philosophy and approach, to develop their own terror cells and hatch their \st{own} \Blue{your} plots. How many \st{more} \Blue{less} ISIS militants \st{are} \Blue{these} in Europe, poised to attack? That's not clear. For now, though, the top priority is tracking down the two men linked directly to Tuesday's terror.
	
	\subsection{Empirical example 2: \textbf{Task} - Fake news detection.
		\textbf{Classifier} - LSTM.}
	
	\textbf{Method}: Ours. \textbf{Origin}: 100\% Fake. \textbf{ADV}: 77\% Real\\
	\st{Man} \Blue{Guy} punctuates high-speed chase with stop at In-N-Out Burger drive-thru Print  [Ed.\st{ - Well, that's a new one.} \Red{Okay, that 's a new one.}]  \st{A} \Blue{One} man is in custody after leading police on a bizarre chase into the east Valley on Wednesday night.  Phoenix police \st{began} \Red{has begun} following the suspect in Phoenix and the pursuit continue\st{d} into the east Valley, but it took a bizarre turn when the suspect stopped at an In-N-Out Burger restaurant’s \st{drive-thru} \Blue{drive-through} near Priest and Ray Roads in Chandler.  The suspect appeared to order food, but then drove away and got out of his pickup truck near Rock Wren Way and Ray Road.  He \st{then ran into a backyard} \Red{ran to the backyard} and tried to \st{get into a house through the back door} \Red{get in the home}.
	
	\textbf{Mehod}: greedy. \textbf{Origin}: 100\% Fake. \textbf{ADV}: 86\% Fake.\\
	\st{Man} \Blue{Guy} punctuates high-speed chase with stop\Blue{ping} at In-N-Out Burger drive-thru Print  [Ed. - Well, that’s a new one.]  A\Blue{nother} man \st{is} \Blue{which} in custody \st{after} \Blue{earlier} leading \st{police} \Blue{officers} on a bizarre chase \st{into} \Blue{out} the \st{east} \Blue{north} Valley on Wednesday night.  Phoenix \st{police} \Blue{arrested} \st{began} \Blue{begun} following \st{the} \Blue{of} suspect\Blue{s} in Phoenix and the \st{pursuit} \Blue{pursuing} continue\st{d} into the east\Blue{ern} Valley, \st{but} \Blue{though} it took a bizarre turn\Blue{ing} when \st{the} \Blue{of} suspect\Blue{s} \st{stopped} \Blue{stopping} at an In-N-Out Burger restaurant’s drive-thru near\Blue{by} Priest and Ray Road\st{s} in Chandler.  The suspect appeared to order food, but then drove away and got out of his pickup truck near Rock Wren Way and Ray Road.  He then ran into a backyard and tried to get into a house through the back door.
	
	\textbf{Method}: gradient method\cite{gong2018adversarial}. \textbf{Origin}: 100\% Fake. \textbf{ADV}: $1-2.5e^{-3}$ Fake.
	
	Man punctuates high-speed chase with stop\Blue{ping} at In-N-Out Burger\Blue{s} drive-thru Print  [Ed. - Well, that’s a new one.]  A man is in custody after leading \st{police} \Blue{arrest} on a bizarre chase into the east Valley on Wednesday night.  Phoenix police began following the suspect in Phoenix and the pursuit continued into \st{the} \Blue{of} \st{east} \Blue{west} Valley, but it took a bizarre \st{turn} \Blue{then} \st{when} \Blue{that} \st{the} \Blue{of} suspect stopped at an In-N-Out Burger restaurant’s drive-thru near\Blue{by} Priest and Ray Roads in Chandler.  The suspect appeared to order\Blue{ing} food, but then drove away and got out of his \st{pickup} \Blue{pick-up} truck near Rock\Blue{s} \st{Wren} \Blue{Chickadee} Way\Blue{s} and Ray Road.  He then ran into a backyard and tried to get into a house through the \st{back} \Blue{again} door.
	
	\subsection{Empirical example 3: \textbf{Task} - Spam filtering. \textbf{Classifier} - WCNN. }
	
	\textbf{Method}: Ours. \textbf{Origin}: 100\% Spam. \textbf{ADV}: 77\% Ham\\
	Become Fit For Life!  HGH is a very complex molecule produced by the anterior lobe of the pituitary gland, which is located at the base of the brain. While it stimulates growth in children, it is important for maintaining \st{a healthy body} \Red{healthy bodies} composition and well-being in adults. It is the primary \st{hormone} \Blue{estrogen} that controls \st{many} \Blue{several} of the body's organs and it stimulates tissue repair,  brain\Blue{s} functions, cell replacement, and enzyme function. Determining the  levels of IGF-1 (Insulin Growth Factor) is how we measure HGH in the body. Receive a younger \st{future} \Blue{potential} with HGH 
	
	\textbf{Method}: Greedy\cite{kuleshov2018adversarial}. \textbf{Origin}: 100\% Spam. \textbf{ADV}: 71\% Ham
	
	Become Fit For Life!  HGH is a \st{very} \Blue{fairly} complex molecule produce\st{d} by the anterior lobe \st{of} \Blue{the} \st{the} \Blue{of} pituitary gland, which \st{is} \Blue{has} located at the base of the brain. While \st{it} \Blue{that} stimulates \st{growth} \Blue{growing} in children, \st{it} \Blue{what} is \st{important} \Blue{significant} for maintaining a\Blue{nother} healthy  \st{body} 
	\Blue{bodies} composition and well-being in adults. It is the \st{primary} \Blue{secondary} \st{hormone} \Blue{progesterone} \st{that} \Blue{could} controls \st{many} \Blue{several} of the body's organs and \st{it} \Blue{that} stimulates tissue repair,  brain\Blue{s} functions, cell replacement, and enzyme function. \st{Determining} \Blue{Determine} the levels of IGF-1 (Insulin Growth Factor) \st{is} \Blue{which} \st{how} \Blue{understand} we measure HGH in \st{the} \Blue{of} body. Receive a younger future with HGH 
	
	\textbf{Method}: Gradient method\cite{gong2018adversarial}. \textbf{Origin}: 100\% Spam. \textbf{ADV}: $1-2.7e^{-5}$ spam
	
	Become Fit For Life!  HGH is a very complex molecule produced by the anterior lobe of the pituitary gland, \st{which} \Blue{that} is \st{located} \Blue{situated} at the base of the brain. While it stimulate\st{s} growth in children, \st{it} \Blue{but} \st{is} \Blue{has} important for maintaining a healthy  \st{body} \Blue{bodies} composition\Blue{s} and well-being in adults. It is the \st{primary} \Blue{secondary} hormone\Blue{s} that controls many of the body's organ\st{s} and \st{it} \Blue{but} stimulate\st{s} tissue\Blue{s} repair,  brain\Blue{s} functions, cell replacement, and enzyme function. Determining the  levels of IGF-1 (Insulin Grow\st{th}\Blue{ing} Factor) is how we measure HGH in \st{the} \Blue{of} body. Receive a younger future with HGH 
	
	\subsection{Empirical Example 4: \textbf{Task} - Spam filtering. \textbf{Classifier} - LSTM.}
	\textbf{Method}: Ours. \textbf{Origin}: 100\% Ham. \textbf{ADV}: 87\% Spam\\
	I've always run jigdo-lite against my own mirror.  It \st{provides} \Blue{offers} \st{two} \Blue{couple} things:  \\
	1) Proves \st{I can} \Red{you are able to} build the ISOs from what I have mirrored locally.\\
	2) Doesn't waste additional bandwidth.  As long as the checksums match what is provided from the official ISO image master\Blue{s} site, I don't see what the difference would be.  Anyone else do this? \st{$:)$} \Blue{\char`^\char`_\char`^}\\
	Will Simon Paillard wrote: \\
	> On Mon, Apr 09, 2007 at 08:43:07AM -0400, Jean-Francois Chevrette wrote: \\
	> > Hi,\\ 
	> > \\
	> > does anyone have a\Blue{nother} straightforward guide on how to use jigdo to build and mirror ISOs? I've been reading both jigdo documentation and debian's \st{webpage} \Blue{web-site} on the subjet and it just won't work. \\
	> \\
	> Maybe with this one :\\ \url{http://www.debian.org/CD/mirroring/\#jigdomirror} \\
	> and the related links ? \\
	>\\  
	> Best regards, 
	
	\textbf{Method}: Greedy\cite{kuleshov2018adversarial}. \textbf{Origin}: 100\% Ham. \textbf{ADV}: 90\% Spam\\
	I've always run jigdo-lite against \st{my} \Blue{myself} own mirror.  It \st{provides} \Blue{offers} \st{two} \Blue{five} things:  \\
	1) Proves \st{I} \Blue{u} can reliably build the ISOs from \st{what} \Blue{something} \st{I} \Blue{im} \st{have} {\color{blue} 've} mirrored locally.\\
	2) Doesn't waste \st{additional} \Blue{extra} bandwidth.  As long as the checksums match\Blue{es} what is provided from the \Blue{un}official ISO image master site, \st{I} \Blue{thats} don't \st{see} \Blue{'ll} what the difference would be.  \st{Anyone} \Blue{Somebody} else do \st{this} \Blue{you}? \st{:)} \Blue{!!}\\
	\st{Will} \Blue{Must} Simon Paillard wrote:\\ 
	> On Mon, Apr 09, 2007 at 08:43:07AM -0400, Jean-Francois Chevrette wrote: \\
	> > Hi, \\
	> > \\
	> > does anyone have a\Blue{nother} straightforward guide on how \st{to} \Blue{able} use jigdo to build  and mirror ISOs? I've been \st{reading} \Blue{writing} \st{both} \Blue{other} jigdo documentation and debian's  web\st{page} on the subjet and it just won't work. \\
	> \\
	> Maybe with this one :\\
	\url{ http://www.debian.org/CD/mirroring/\#jigdomirror} \\
	> and the related links ? \\
	>\\  
	> Best regards, 
	
	\textbf{Method}: Gradient method\cite{gong2018adversarial}. \textbf{Origin}: 100\% Ham. \textbf{ADV}: $1-2.2e^{-15}$ Ham\\
	I've always run jigdo-lite against \st{my} \Blue{myself} own mirror.  \st{It} \Blue{What} provides two things:  \\
	1) Proves \st{I} \Blue{you} can reliably build\Blue{ing} the ISOs from what I have mirrored locally.\\
	2) Doesn't waste additional bandwidth.  As long as the checksums match what is provided from the \Blue{un}official ISO image master site, \st{I} \Blue{u} don't see what the difference \st{would} \Blue{could} be.  Anyone else do this? \st{:)} \Blue{:-)}\\
	Will Simon Paillard wrote: \\
	> On Mon, Apr 09, 2007 at 08:43:07AM -0400, Jean-Francois Chevrette wrote: \\
	> > Hi, \\
	> > \\
	> > does anyone have a straightforward guide on \st{how} \Blue{what} \st{to} \Blue{able} use jigdo to build and mirror ISOs? I've been reading both jigdo documentation and debian's  \st{webpage} \Blue{web-site} on the subjet and it just won't work. \\
	> 
	> Maybe with this one :\\ \url{http://www.debian.org/CD/mirroring/\#jigdomirror} \\
	> and the related links ? \\
	>\\  
	> Best regards, 
	
	\subsection{Empirical Example 5: \textbf{Task} - Sentiment analysis. \textbf{Classifier} - CNN.}
	
	\textbf{Method}: Ours. \textbf{Origin}: 100\% Positive. \textbf{ADV}: 93\% Negative\\
	This Starbucks location is located in the Bally's Grand Bazaar Shops. It's open 24/7 and it is huge. There is plenty of seating. Most of the seating is stadium type seating with benches. They also have an out door patio. The staff is very friendly and attentive to the guests. I do notice that they are under staffed sometimes when they are busy. They \Red{'ll} get your drinks \st{out} pretty fast \st{though}. Also, this \st{location} \Red{place} is not owned by the \st{casino} \Red{property} so they \st{don't} \Red{do n't} charge outrageous prices \st{like the location} \Red{as a place} on \st{the} \Red{an} Linq promenade \st{does}.  Definitely one of my favorite Starbucks stores. Stop by if your on the Strip.
	
	\textbf{Method}: Greedy\cite{kuleshov2018adversarial}. \textbf{Origin}: 100\% Positive. \textbf{ADV}: 74\% Negative\\
	This Starbucks location \st{is} \Blue{be} located in the Bally's Grand Bazaar Shops. It's open 24/7 and \st{it} \Blue{nothing} \st{is} \Blue{be} huge. \st{There} \Blue{Nothing} is plenty \st{of} \Blue{the} seating. \st{Most} \Blue{Extremely} of \st{the} \Blue{of} seating \st{is} \Blue{has} stadium type\Blue{s} \st{seating} \Blue{seats} with benches. \st{They} \Blue{Have} \st{also} \Blue{will} \st{have} \Blue{never} an out door patio. The staff is very friendly and attentive to the guests. I do notice that they are under staffed sometimes when they are busy. They get your drinks out pretty fast though. Also, this location is not owned by the casino so they don't charge outrageous prices like the location on the Linq promenade does.  Definitely one of my favorite Starbucks stores. Stop by if your on the Strip.
	
	\textbf{Method}: Gradient method\cite{gong2018adversarial}.
	\textbf{Origin}: 100\% Positive. \textbf{ADV}: $1-6.9e^{-12}$ Positive
	\st{This} \Blue{It} \st{Starbucks} \Blue{Mcdonalds} location is located in the Bally's Grand Bazaar Shops. It's open 24/7 and it is huge. There is plenty of seating. \st{Most} \Blue{Many} \st{of} \Blue{the} the seating is stadium type seating with benches. They also \st{have} {\color{blue} 've} an \st{out} \Blue{up} door patio. The staff is very friendly and attentive to the guests. I do notice that they are under staffed sometimes when they are busy. They get\Blue{ting} your drinks out pretty fast though. Also, this location is not owned by \st{the} \Blue{of} casino \st{so} \Blue{too} they don't charge outrageous prices \st{like} \Blue{think} the location on \st{the} \Blue{of} Linq \st{promenade} \Blue{seafront} does.  Definitely one of my favorite Starbucks stores. Stop by \st{if} \Blue{unless} your on the Strip.

	\subsection{Empirical Example 6: \textbf{Task} - Sentiment analysis. \textbf{Classifier} - LSTM.}
	
	\textbf{Method}: Ours. \textbf{Origin}: 100\% Positive. \textbf{ADV}: 93\% Negative\\
	I suppose I should write a review here since my little Noodle-oo is currently serving as their spokes dog in the photos. We both love Scooby Do's. They treat my little butt-faced dog like a prince and are receptive to correcting anything about the cut that I perceive as being weird. Like that funny poofy pompadour. Mohawk it out, yo. Done. In like five seconds my little man was looking fabulous and bad ass. Not something easily accomplished with a prancing pup that literally chases butterflies through tall grasses. (He ended up looking like a little lamb as the cut grew out too. So adorable.)  The shampoo they use here is also amazing. Noodles usually smells like tacos (a combination of beef stank and corn chips) but after getting back from the Do's, he smelled like Christmas morning! Sugar and spice and everything nice instead of frogs and snails and puppy dog tails. He's got some gender identity issues to deal with.   \st{The pricing is also cheaper than some of the big name conglomerates out there} \Red{The price is cheaper than some of the big names below}. I'm talking to you Petsmart! I've taken my other pup to Smelly Dog before, but unless I need dog sitting play time after the cut, I'll go with Scooby's. They genuinely seem to like my little Noodle monster.
	
	\textbf{Method}: Greedy\cite{kuleshov2018adversarial}. \textbf{Origin}: 100\% Positive. \textbf{ADV}: 88\% Negative\\
	I suppose I should write a review here since my little Noodle-oo is currently serving as their spokes dog in the photos. We both love Scooby Do's. They treat my little butt-faced dog like a prince and are receptive to correcting anything about the cut that I perceive as being weird. Like that \st{funny} \Blue{humorous} poofy pompadour. Mohawk it out, yo. Done. In like five seconds my little \Blue{wo}man was looking fabulous and bad ass. Not something easily accomplished with a prancing pup that literally chases butterflies through tall grasses. (He ended up looking like a little lamb as the cut grew out too. So adorable.)  The shampoo they use here is also amazing. Noodles usually smells like tacos (a combination \st{of} \Blue{between} beef stank and corn chips) but after getting back from the Do's, he smelled like Christmas morning! Sugar and spice and everything nice instead of frogs and snails and puppy dog tails. He's got some gender identity issues to deal with. The pricing is also cheaper than some of the big name conglomerates out there. I'm talking to you Petsmart! I've taken my other pup to Smelly Dog before, but unless I need dog sitting play time after the cut, I'll go with Scooby's. They genuinely seem to like my little Noodle monster.
	
	\textbf{Method}: Gradient method\cite{gong2018adversarial}.
	\textbf{Origin}: 100\% Positive. \textbf{ADV}: 93\% Negative\\
	I suppose I should \st{write} \Blue{Write} a review here since my little Noodle-oo is currently serving as their spokes dog in the photos.   We both love Scooby Do's. They \st{treat} \Blue{cure} my little butt-faced dog like a\Blue{nother} \st{prince} \Blue{knight} \st{and} \Blue{but} are receptive to correcting anything about the cut that I perceive \st{as} \Blue{that} being weird. Like that funny poofy pompadour. Mohawk it out, yo. Done. In like \st{five} \Blue{eleven} \st{seconds} \Blue{secs} my little man was looking fabulous and bad ass. Not something \st{easily} \Blue{readily} accomplished with a \st{prancing} \Blue{strutting} pup that literally chases butterflies through tall grasses. (He ended up looking like a little \st{lamb} \Blue{beef} as \st{the} \Blue{The} cut grew out too. So adorable.)  The shampoo they use here is also amazing. Noodles usually smells like \st{tacos} \Blue{quesadillas} (a combination of beef stank and corn chips) but after getting back from the Do's, he smelled like Christmas morning! Sugar and \st{spice} \Blue{cumin} and everything nice instead of frogs and snails and puppy dog tails. He's got \st{some} \Blue{those} \st{gender} \Blue{sexuality} identity \st{issues} \Blue{difficulties} to \st{deal} \Blue{contract} with.   The pricing is also cheaper than some of the \st{big} \Blue{huge} name conglomerates out there. I'm talking to you Petsmart! I've \st{taken} \Blue{brought} my other pup to Smelly Dog before, but unless I need dog sitting play time after the cut, I'll go with Scooby's. They \st{genuinely} \Blue{nonetheless} seem to like my little Noodle monster.

%% file: arxiv.bbl
\begin{thebibliography}{10}
\providecommand{\url}[1]{#1}
\csname url@samestyle\endcsname
\providecommand{\newblock}{\relax}
\providecommand{\bibinfo}[2]{#2}
\providecommand{\BIBentrySTDinterwordspacing}{\spaceskip=0pt\relax}
\providecommand{\BIBentryALTinterwordstretchfactor}{4}
\providecommand{\BIBentryALTinterwordspacing}{\spaceskip=\fontdimen2\font plus
\BIBentryALTinterwordstretchfactor\fontdimen3\font minus
  \fontdimen4\font\relax}
\providecommand{\BIBforeignlanguage}[2]{{%
\expandafter\ifx\csname l@#1\endcsname\relax
\typeout{** WARNING: IEEEtran.bst: No hyphenation pattern has been}%
\typeout{** loaded for the language `#1'. Using the pattern for}%
\typeout{** the default language instead.}%
\else
\language=\csname l@#1\endcsname
\fi
#2}}
\providecommand{\BIBdecl}{\relax}
\BIBdecl

\bibitem{dalvi2004adversarial}
N.~Dalvi, P.~Domingos, S.~Sanghai, D.~Verma \emph{et~al.}, ``Adversarial
  classification,'' in \emph{Proceedings of the tenth ACM SIGKDD international
  conference on Knowledge discovery and data mining}.\hskip 1em plus 0.5em
  minus 0.4em\relax ACM, 2004, pp. 99--108.

\bibitem{lowd2005adversarial}
D.~Lowd and C.~Meek, ``Adversarial learning,'' in \emph{Proceedings of the
  eleventh ACM SIGKDD international conference on Knowledge discovery in data
  mining}.\hskip 1em plus 0.5em minus 0.4em\relax ACM, 2005, pp. 641--647.

\bibitem{biggio2017wild}
B.~Biggio and F.~Roli, ``Wild patterns: Ten years after the rise of adversarial
  machine learning,'' \emph{arXiv preprint arXiv:1712.03141}, 2017.

\bibitem{szegedy2013intriguing}
C.~Szegedy, W.~Zaremba, I.~Sutskever, J.~Bruna, D.~Erhan, I.~Goodfellow, and
  R.~Fergus, ``Intriguing properties of neural networks,'' \emph{arXiv preprint
  arXiv:1312.6199}, 2013.

\bibitem{goodfellow2014explaining}
I.~J. Goodfellow, J.~Shlens, and C.~Szegedy, ``Explaining and harnessing
  adversarial examples,'' \emph{arXiv preprint arXiv:1412.6572}, 2014.

\bibitem{moosavi2016deepfool}
S.~M. Moosavi~Dezfooli, A.~Fawzi, and P.~Frossard, ``Deepfool: a simple and
  accurate method to fool deep neural networks,'' in \emph{Proceedings of 2016
  IEEE Conference on Computer Vision and Pattern Recognition (CVPR)}, no.
  EPFL-CONF-218057, 2016.

\bibitem{papernot2016distillation}
N.~Papernot, P.~McDaniel, X.~Wu, S.~Jha, and A.~Swami, ``Distillation as a
  defense to adversarial perturbations against deep neural networks,'' in
  \emph{Security and Privacy (SP), 2016 IEEE Symposium on}.\hskip 1em plus
  0.5em minus 0.4em\relax IEEE, 2016, pp. 582--597.

\bibitem{carlini2017towards}
N.~Carlini and D.~Wagner, ``Towards evaluating the robustness of neural
  networks,'' in \emph{IEEE Symposium on Security and Privacy}, 2017, pp.
  39--57.

\bibitem{evtimov2017robust}
I.~Evtimov, K.~Eykholt, E.~Fernandes, T.~Kohno, B.~Li, A.~Prakash, A.~Rahmati,
  and D.~Song, ``Robust physical-world attacks on deep learning models,''
  \emph{arXiv preprint arXiv:1707.08945}, vol.~1, 2017.

\bibitem{chen2017zoo}
P.-Y. Chen, H.~Zhang, Y.~Sharma, J.~Yi, and C.-J. Hsieh, ``{ZOO}: Zeroth order
  optimization based black-box attacks to deep neural networks without training
  substitute models,'' in \emph{ACM Workshop on Artificial Intelligence and
  Security}, 2017, pp. 15--26.

\bibitem{chen2017ead}
P.-Y. Chen, Y.~Sharma, H.~Zhang, J.~Yi, and C.-J. Hsieh, ``{EAD}: elastic-net
  attacks to deep neural networks via adversarial examples,'' \emph{AAAI},
  2018.

\bibitem{su2018robustness}
D.~Su, H.~Zhang, H.~Chen, J.~Yi, P.-Y. Chen, and Y.~Gao, ``Is robustness the
  cost of accuracy?--a comprehensive study on the robustness of 18 deep image
  classification models,'' \emph{ECCV}, 2018.

\bibitem{madry2017towards}
A.~Madry, A.~Makelov, L.~Schmidt, D.~Tsipras, and A.~Vladu, ``Towards deep
  learning models resistant to adversarial attacks,'' \emph{arXiv preprint
  arXiv:1706.06083}, 2017.

\bibitem{athalye2018obfuscated}
A.~Athalye, N.~Carlini, and D.~Wagner, ``Obfuscated gradients give a false
  sense of security: Circumventing defenses to adversarial examples,''
  \emph{arXiv preprint arXiv:1802.00420}, 2018.

\bibitem{papernot2016crafting}
N.~Papernot, P.~McDaniel, A.~Swami, and R.~Harang, ``Crafting adversarial input
  sequences for recurrent neural networks,'' in \emph{Military Communications
  Conference, MILCOM 2016-2016 IEEE}.\hskip 1em plus 0.5em minus 0.4em\relax
  IEEE, 2016, pp. 49--54.

\bibitem{li2016understanding}
J.~Li, W.~Monroe, and D.~Jurafsky, ``Understanding neural networks through
  representation erasure,'' \emph{arXiv preprint arXiv:1612.08220}, 2016.

\bibitem{ebrahimi2017hotflip}
J.~Ebrahimi, A.~Rao, D.~Lowd, and D.~Dou, ``Hotflip: White-box adversarial
  examples for {NLP},'' \emph{arXiv preprint arXiv:1712.06751}, 2017.

\bibitem{gong2018adversarial}
Z.~Gong, W.~Wang, B.~Li, D.~Song, and W.-S. Ku, ``Adversarial texts with
  gradient methods,'' \emph{arXiv preprint arXiv:1801.07175}, 2018.

\bibitem{kuleshov2018adversarial}
\BIBentryALTinterwordspacing
V.~Kuleshov, S.~Thakoor, T.~Lau, and S.~Ermon, ``Adversarial examples for
  natural language classification problems,'' 2018. [Online]. Available:
  \url{https://openreview.net/forum?id=r1QZ3zbAZ}
\BIBentrySTDinterwordspacing

\bibitem{yang2018greedy}
P.~Yang, J.~Chen, C.-J. Hsieh, J.-L. Wang, and M.~I. Jordan, ``Greedy attack
  and {G}umbel attack: Generating adversarial examples for discrete data,''
  \emph{arXiv preprint arXiv:1805.12316}, 2018.

\bibitem{jia2017adversarial}
R.~Jia and P.~Liang, ``Adversarial examples for evaluating reading
  comprehension systems,'' \emph{arXiv preprint arXiv:1707.07328}, 2017.

\bibitem{miyato2016adversarial}
T.~Miyato, A.~M. Dai, and I.~Goodfellow, ``Adversarial training methods for
  semi-supervised text classification,'' \emph{arXiv preprint
  arXiv:1605.07725}, 2016.

\bibitem{samanta2017towards}
S.~Samanta and S.~Mehta, ``Towards crafting text adversarial samples,''
  \emph{arXiv preprint arXiv:1707.02812}, 2017.

\bibitem{liang2017deep}
B.~Liang, H.~Li, M.~Su, P.~Bian, X.~Li, and W.~Shi, ``Deep text classification
  can be fooled,'' \emph{arXiv preprint arXiv:1704.08006}, 2017.

\bibitem{yao2017automated}
Y.~Yao, B.~Viswanath, J.~Cryan, H.~Zheng, and B.~Y. Zhao, ``Automated
  crowdturfing attacks and defenses in online review systems,'' in \emph{ACM
  Conference on Computer and Communications Security}, 2017, pp. 1143--1158.

\bibitem{gao2018black}
J.~Gao, J.~Lanchantin, M.~L. Soffa, and Y.~Qi, ``Black-box generation of
  adversarial text sequences to evade deep learning classifiers,'' \emph{arXiv
  preprint arXiv:1801.04354}, 2018.

\bibitem{alzantot2018generating}
M.~Alzantot, Y.~Sharma, A.~Elgohary, B.-J. Ho, M.~Srivastava, and K.-W. Chang,
  ``Generating natural language adversarial examples,'' \emph{arXiv preprint
  arXiv:1804.07998}, 2018.

\bibitem{wong2017dancin}
C.~Wong, ``Dancin seq2seq: Fooling text classifiers with adversarial text
  example generation,'' \emph{arXiv preprint arXiv:1712.05419}, 2017.

\bibitem{zhao2017generating}
Z.~Zhao, D.~Dua, and S.~Singh, ``Generating natural adversarial examples,''
  \emph{ICLR; arXiv preprint arXiv:1710.11342}, 2018.

\bibitem{cheng2018seq2sick}
M.~Cheng, J.~Yi, H.~Zhang, P.-Y. Chen, and C.-J. Hsieh, ``Seq2sick: Evaluating
  the robustness of sequence-to-sequence models with adversarial examples,''
  \emph{arXiv preprint arXiv:1803.01128}, 2018.

\bibitem{ribeiro2018semantically}
M.~T. Ribeiro, S.~Singh, and C.~Guestrin, ``Semantically equivalent adversarial
  rules for debugging {NLP} models,'' in \emph{Proceedings of the 56th Annual
  Meeting of the Association for Computational Linguistics (Volume 1: Long
  Papers)}, vol.~1, 2018, pp. 856--865.

\bibitem{belinkov2017synthetic}
Y.~Belinkov and Y.~Bisk, ``Synthetic and natural noise both break neural
  machine translation,'' \emph{arXiv preprint arXiv:1711.02173}, 2017.

\bibitem{mrkvsic2016counter}
N.~Mrk{\v{s}}i{\'c}, D.~O. S{\'e}aghdha, B.~Thomson, M.~Ga{\v{s}}i{\'c},
  L.~Rojas-Barahona, P.-H. Su, D.~Vandyke, T.-H. Wen, and S.~Young,
  ``Counter-fitting word vectors to linguistic constraints,'' \emph{arXiv
  preprint arXiv:1603.00892}, 2016.

\bibitem{frank1956algorithm}
M.~Frank and P.~Wolfe, ``An algorithm for quadratic programming,'' \emph{Naval
  Research Logistics (NRL)}, vol.~3, no. 1-2, pp. 95--110, 1956.

\bibitem{narayanan1997submodular}
H.~Narayanan, \emph{Submodular functions and electrical networks}.\hskip 1em
  plus 0.5em minus 0.4em\relax Elsevier, 1997, vol.~54.

\bibitem{fujishige2005submodular}
S.~Fujishige, \emph{Submodular functions and optimization}.\hskip 1em plus
  0.5em minus 0.4em\relax Elsevier, 2005, vol.~58.

\bibitem{schrijver2003combinatorial}
A.~Schrijver, \emph{Combinatorial optimization: polyhedra and
  efficiency}.\hskip 1em plus 0.5em minus 0.4em\relax Springer Science \&
  Business Media, 2003, vol.~24.

\bibitem{nemhauser1978analysis}
G.~L. Nemhauser, L.~A. Wolsey, and M.~L. Fisher, ``An analysis of
  approximations for maximizing submodular set functions—i,''
  \emph{Mathematical Programming}, vol.~14, no.~1, pp. 265--294, 1978.

\bibitem{nutini2015coordinate}
J.~Nutini, M.~Schmidt, I.~Laradji, M.~Friedlander, and H.~Koepke, ``Coordinate
  descent converges faster with the {G}auss-{S}outhwell rule than random
  selection,'' in \emph{International Conference on Machine Learning}, 2015,
  pp. 1632--1641.

\bibitem{bilmes2017deep}
J.~Bilmes and W.~Bai, ``Deep submodular functions,'' \emph{arXiv preprint
  arXiv:1701.08939}, 2017.

\bibitem{kim2014convolutional}
Y.~Kim, ``Convolutional neural networks for sentence classification,''
  \emph{arXiv preprint arXiv:1408.5882}, 2014.

\bibitem{wieting2015paraphrase}
J.~Wieting, M.~Bansal, K.~Gimpel, K.~Livescu, and D.~Roth, ``From paraphrase
  database to compositional paraphrase model and back,'' \emph{arXiv preprint
  arXiv:1506.03487}, 2015.

\bibitem{wieting-17-millions}
J.~Wieting and K.~Gimpel, ``Pushing the limits of paraphrastic sentence
  embeddings with millions of machine translations,'' in \emph{arXiv preprint
  arXiv:1711.05732}, 2017.

\bibitem{kusner2015word}
M.~Kusner, Y.~Sun, N.~Kolkin, and K.~Weinberger, ``From word embeddings to
  document distances,'' in \emph{International Conference on Machine Learning},
  2015, pp. 957--966.

\bibitem{lei2016coordinate}
Q.~Lei, K.~Zhong, and I.~S. Dhillon, ``Coordinate-wise power method,'' in
  \emph{Advances in Neural Information Processing Systems}, 2016, pp.
  2064--2072.

\bibitem{lei2017doubly}
Q.~Lei, I.~E. Yen, C.-y. Wu, I.~S. Dhillon, and P.~Ravikumar, ``Doubly greedy
  primal-dual coordinate descent for sparse empirical risk minimization,'' in
  \emph{Proceedings of the 34th International Conference on Machine
  Learning-Volume 70}.\hskip 1em plus 0.5em minus 0.4em\relax JMLR. org, 2017,
  pp. 2034--2042.

\bibitem{hochreiter1997long}
S.~Hochreiter and J.~Schmidhuber, ``Long short-term memory,'' \emph{Neural
  computation}, vol.~9, no.~8, pp. 1735--1780, 1997.

\bibitem{DBLP:journals/corr/ZhangZL15}
\BIBentryALTinterwordspacing
X.~Zhang, J.~J. Zhao, and Y.~LeCun, ``Character-level convolutional networks
  for text classification,'' \emph{CoRR}, vol. abs/1509.01626, 2015. [Online].
  Available: \url{http://arxiv.org/abs/1509.01626}
\BIBentrySTDinterwordspacing

\bibitem{McIntire2017Fake}
G.~McIntire, ``Fake news dataset,''
  \url{https://github.com/GeorgeMcIntire/fake_ real_news_dataset}, 2017.

\bibitem{mikolov2013efficient}
T.~Mikolov, K.~Chen, G.~Corrado, and J.~Dean, ``Efficient estimation of word
  representations in vector space,'' \emph{arXiv preprint arXiv:1301.3781},
  2013.

\bibitem{gal2016dropout}
Y.~Gal and Z.~Ghahramani, ``Dropout as a bayesian approximation: Representing
  model uncertainty in deep learning,'' in \emph{international conference on
  machine learning}, 2016, pp. 1050--1059.

\end{thebibliography}
